%% file: main.tex
\documentclass[journal]{IEEEtran}
\IEEEoverridecommandlockouts
\usepackage{cite}
\usepackage{amsmath,amsthm,amssymb,amsfonts}
\usepackage{algorithmic}
\usepackage{graphicx}
\usepackage{textcomp}
\usepackage{xcolor}
\usepackage{bm}
\usepackage{multirow}
\usepackage{subfigure}
\usepackage{mathrsfs}
\usepackage[ruled,linesnumbered]{algorithm2e}
\usepackage{pdfpages}
\usepackage{lettrine}
\DeclareMathOperator*{\argmax}{arg\,max}

\newtheorem{theorem}{Theorem}
\newtheorem{lemma}{Lemma}
\newtheorem{corollary}{Corollary}

\allowdisplaybreaks[4]
\usepackage{booktabs, makecell, tabularx}
\usepackage[para]{threeparttable}
\usepackage{tablefootnote}

\newcommand{\blus}{}
\newcommand{\blums}{}

\newcommand{\blut}{}
\newcommand{\blumt}{}

\usepackage{nccmath}
\newenvironment{Eequation}{\useshortskip\leavevmode\vspace*{-.85\baselineskip}\equation}{\endequation}
\usepackage{comment}

\def\BibTeX{{\rm B\kern-.05em{\sc i\kern-.025em b}\kern-.08em
    T\kern-.1667em\lower.7ex\hbox{E}\kern-.125emX}}
\begin{document}

\title{Noise Distribution Decomposition based Multi-Agent Distributional Reinforcement Learning\\

\author{Wei Geng, Baidi Xiao, Rongpeng Li, Ning Wei, Dong Wang, and Zhifeng Zhao\\
% \thanks{W. Geng is with Hangzhou Institute for Advanced Study, UCAS (email: gengwei21@mails.ucas.ac.cn). }\
\thanks{This work was supported in part by the National Key Research and Development Program of China under Grant 2024YFE0200600, in part by the National Natural Science Foundation of China under Grant 62071425, in part by the Zhejiang Key Research and Development Plan under Grant 2022C01093, in part by the Zhejiang Provincial Natural Science Foundation of China under Grant LR23F010005, and in part by Zhejiang Lab under Grant 3400-31299. (Corresponding Authors: \textit{Ning Wei} and \textit{Zhifeng Zhao})}
\thanks{W. Geng, N. Wei and Z. Zhao are with Zhejiang Lab (email: gengwei21@mails.ucas.ac.cn, \{weining, zhaozf\}@zhejianglab.com). }
\thanks{B. Xiao and R. Li are with the College of Information Science and Electronic Engineering, Zhejiang University (email: \{xiaobaidi, lirongpeng\}@zju.edu.cn). }
\thanks{D. Wang is with the Research Institute of China Telecom (email: wangd5@chinatelecom.cn).}
}}

\maketitle

\begin{abstract}
Generally, Reinforcement Learning (RL) agent updates its policy by repetitively interacting with the environment, contingent on the received rewards to observed states and undertaken actions. However, the environmental disturbance, commonly leading to noisy observations (e.g., rewards and states), could significantly shape the performance of agent. Furthermore, the learning performance of Multi-Agent Reinforcement Learning (MARL) is more susceptible to noise due to the interference among intelligent agents. Therefore, it becomes imperative to revolutionize the design of MARL, so as to capably ameliorate the annoying impact of noisy rewards. In this paper, we propose a novel decomposition-based multi-agent distributional RL method by approximating the globally shared noisy reward by a Gaussian Mixture Model (GMM) and decomposing it into the combination of individual distributional local rewards, with which each agent can be updated locally through distributional RL. Moreover, a Diffusion Model (DM) is leveraged for reward generation in order to mitigate the issue of costly interaction expenditure for learning distributions. Furthermore, the monotonicity of the reward distribution decomposition is theoretically validated under nonnegative weights and increasing distortion risk function, while the design of the loss function is carefully calibrated to avoid decomposition ambiguity. We also verify the effectiveness of the proposed method through extensive simulation experiments with noisy rewards. Besides, different risk-sensitive policies are evaluated in order to demonstrate the superiority of distributional RL in different MARL tasks. 
\end{abstract}


\section{Introduction}
\lettrine[lines=2]{R}EINFORCEMENT learning (RL)\cite{sutton1998introduction} has significant applications in various domains, including game AI training\cite{lample2017playing}, robot control\cite{singh2022reinforcement}, and large language models' optimization\cite{du2023guiding}. 
With the rise of Deep Learning (DL)\cite{deeplearning} technology and its remarkable achievements in various fields, the integration of Deep Neural Networks (DNNs) and RL \cite{deeplearningmatters} has become a hot topic in research. 
The accompanied tremendous success of deep RL has prompted researchers to shift their focus to the field of Multi-Agent Systems (MAS)\cite{mas}, 
and given rise to Multi-Agent Reinforcement Learning (MARL). 
For most cooperative MARL tasks, it is crucial to develop appropriate means to leverage a global reward to train decentralized action value (i.e., $Q$-value) or state value (i.e., $V$-value) networks, and update the policies accordingly, especially for settings with partial observations and constrained communications \cite{huttenrauch2017guided, sunehag2017value}. In particular, some value decomposition methods \cite{sunehag2017value,rashid2020monotonic,rashid2020weighted,son2019qtran} are proposed. 

Nevertheless, environmental noise widely exists in practice due to internal and external factors (e.g., electronic noise from sensors, the influences of temperature, pressure, and illumination) \cite{jia2020uav}, and most RL algorithms are vulnerable in noisy scenarios with unstable performance \cite{wang2020reinforcement}. Some off-the-shelf MARL algorithms also face troublesome convergence issues and experience severely deteriorated performance under noisy scenarios, especially for those adversarial tasks where agents might suffer malicious interference from opponents\cite{kilinc2018multi}. 
The distributional RL \cite{bellemare2017distributional}, which models the action value or state value as a distribution, emerges as a promising solution for anti-noise.
This distribution can effectively model the randomness of a system or environment, and distributional RL can even surpass human performance in many scenarios
\blus{where the optimization must go beyond considering the expected return values and evaluate the distribution of certain performance metrics/rewards \cite{distributional_utility}.}
For instance, \blus{in the field of communication, distributional RL demonstrates its advantages in the robust network optimization that accounts for the outage probabilities \cite{robust_optimization}, the management of slice Service Level Agreements (SLAs) defined by availability and performance thresholds \cite{end_to_end_slicing}, and resource allocation strategies that consider risk in decision-making processes \cite{Service_Provider_Revenue}. Meanwhile, in game AI, }\cite{bellemare2017distributional,dabney2018distributional,dabney2018implicit} demonstrate the superior performance of distributional RL in dozens of tasks about Atari.
However, incorporating value decomposition into classic distributional RL for noisy MARL tasks is rather counter-effective, considering the multi-folded detrimental effects. 
Firstly, it is typically hard to train a quantile network \cite{bellemare2017distributional, dabney2018implicit}, which aims to learn a specific Cumulative Distribution Function (CDF) of action or state value distribution. Therefore, it usually takes distributional RL algorithms more time to converge. Secondly, making specific assumptions about the exact type of global distribution lacks robustness. As not all types of distribution are additive or support linear operations, changing the types of individual distributions through temporal-difference (TD) update also increases model complexity \cite{TDbase}. Additionally, decomposing the global distribution into individual distributions is erratic and unstable \cite{pmlr-v139-sun21c}, as slight changes in the global distribution usually cause individual distributions to vibrate drastically. Finally, the interaction cost between agents and environment, as well as the learning cost, will significantly increase \cite{BIRMAN2022133}.

In this paper, we propose the Noise Distribution Decomposition (NDD) MARL method, wherein the ideas of distributional RL and value decomposition are jointly leveraged to ameliorate the flexibility of distributed agents to noisy rewards. 
The utilization of distributional RL enables the agent to remap quantiles using the distortion risk function, thereby better taking account of the distribution over returns and potentially generating distinctive policies with different risk-sensitive functions \cite{dabney2018implicit}.
Meanwhile, in order to tackle the difficulties of incorporating value decomposition into classic distributional RL for MARL tasks, we successfully leverage a parametric Gaussian Mixture Model (GMM) \cite{reynolds2009gaussian} and establish a viable decomposition means with calibrated Wasserstein-metric-based \cite{dabney2018distributional} loss function. Furthermore, inspired by the astonishing representation capabilities of Diffusion Models (DM)\cite{DDPM,DMvsGAN,Palette,GLIDE, DALLE2}, we also incorporate DM to model the noise distribution to reduce the interaction cost. \blus{Specifically, we use DM to generate additional data based on existing interaction data from agents, and then proportionally combine the generated data with the raw data to train policies in NDD.}
The main contributions of our work can be summarized as follows:
\begin{itemize}
\item We introduce NDD on top of distributional RL and value decomposition in MARL. For the ensemble of agents, NDD approximates the distribution of the noisy global reward by GMM and obtains the local distributional value functions indirectly for yielding more stable individual policies. For a single agent, NDD also contributes to more conveniently learning the action value distribution even with partial observation.
\item We extend distributional RL to the multi-agent domain and introduce distortion risk functions, which facilitate the determination to adopt adventurous or conservative strategies by adjusting action value distributions, to provide more flexibility than classical RL with expectation.
\item We carefully calibrate a Wasserstein-metric-based loss function to learn the accurate distribution corresponding to each agent and prove the consistency between the globally optimal action and locally optimal actions theoretically.
\item To alleviate the high cost of environment interaction required for learning distributions, we introduce DM to augment the data for ameliorating the data scarcity issue \blumt{and prove the bounded expectations of the generation and approximation errors in the augmented data, which can be a reference for the trade off between interaction cost and errors. }
% Based on the proofs, we elaborate on the trade-off between interaction costs and errors.}
\end{itemize}

In the rest of this paper, we present the related work of RL and MARL in Section \ref{sectionPriorWorks}. Afterward, based on the necessary notations and preliminaries in Section \ref{sectionProblem}, we describe our proposed method in Section \ref{sectionNDD}. Subsequently, numerous simulation results are shown to validate the effectiveness and superiority of our method in Section \ref{sectionRes}. Finally, Section \ref{sectionCon} concludes the paper \blut{and lists the future work}.

\section{Related Work}\label{sectionPriorWorks}
According to the cooperativeness between their reward functions, MARL can be categorized into three types (i.e., fully cooperative, fully competitive, and mixed tasks) \cite{marloverview}. For fully cooperative tasks, the reward functions of agents are identical, signifying that all agents are striving to achieve a common goal \cite{boxpushing, Distributed_Q_Learning}. Algorithms targeted at this type of task can integrate the computational and learning capabilities of MAS, and potentially achieve gains beyond simple summation \cite{marloverview}. For fully competitive tasks, the reward functions are opposite, typically involving two completely antagonistic agents in the environment, while the agents aim to maximize their rewards and minimize the opponent's rewards, with Minimax-Q \cite{MarkovRL} being a representative algorithm. In mixed tasks, the relationship of agents could be stochastic, leading to complicated reward functions of agents and making this model suitable for self-interested agents \cite{marloverview}. Generally, the resolution of such tasks is often associated with the concept of equilibrium in game theory. In this paper, we primarily focus on the first, fully cooperative setting, since the related algorithms we will mention later are more conducive to implementation, migration, and expansion to more large-scale MAS \cite{marloverview, cooperativedistributed} compared with the other two types. On the other hand, the cooperative algorithm has a close connection with distributed optimization \cite{cooperativedistributed}. Therefore, efficient optimization techniques such as decentralized training, asynchronous training \cite{asynchronousconsensus}, and communication learning \cite{foerster2016learning, sukhbaatar2016learning} can be introduced to tackle issues in cooperative learning.

However, there are still many unresolved issues for cooperative algorithms. First, in a multi-agent environment, each agent not only needs to consider its own actions and rewards but also must take account of the influences of the actions undertaken by other agents \cite{marloverview}. Such an intricate process of interaction and interconnection results in constantly undergoing changes in a non-stationary environment, making the actions and strategy choices among agents interdependent. Therefore, the updated strategies of other agents impede each agent from approximating accurate action or state value functions for making a self-optimized strategy. To solve this issue, based on Deep Q-Network (DQN)\cite{mnih2013playing}, Castaneda \textit{et al.} \cite{castaneda2016deep} propose to modify the value functions and reward functions to fit the inter-dependence among agents. Diallo \textit{et al.} \cite{drlDoublesPongGame} introduce a parallel computing mechanism into DQN to expedite the convergence in non-stationary environments. Foerster \textit{et al.} \cite{foerster2016learning} focus on improving the experience replay to better adapt to continuously changing non-stationary environments.
Another important issue in cooperative MARL lies in the limitation of partial observation. That is, agents can only make near-optimal decisions based on their own partially observable information.
In the collaborative task with local observations and globally shared rewards, value decomposition methods are often adopted to effectively learn the different contributions of agents from global rewards.
For example, VDN \cite{sunehag2017value} decomposes the global value function into the sum of individual value functions. Based on VDN, QMIX\cite{rashid2020monotonic} further generalizes the way of value decomposition through sophisticated transformation, and puts forward an important assumption on the monotonicity of the global action value function with respect to local ones. Focusing on the further exploration of the joint action, Weighted QMIX\cite{rashid2020weighted} and QTRAN\cite{son2019qtran} discuss the relationship between individual optimal actions and the global optimal joint action.
Nevertheless, all these value decomposition algorithms \cite{sunehag2017value, rashid2020monotonic, rashid2020weighted, son2019qtran} suffer from the adoption of over-simplistic expectation operation, which is not sufficient to describe the noisy environment that agents may encounter. In a nutshell, it necessitates a more comprehensive characterization of the environment, and developing novel decomposition means to explore adventurous or conservative strategies. Coincidentally, it falls into the scope of distributional RL \cite{bellemare2017distributional, dabney2018distributional, dabney2018implicit} and distortion risk functions \cite{dabney2018implicit}.

Primarily adopted in single-agent settings, distributional RL\cite{bellemare2017distributional, dabney2018distributional, dabney2018implicit} comprehensively characterizes the fluctuation by replacing the expectation of state or action values with the corresponding distributions and yields appealing effects. For example, C51\cite{bellemare2017distributional} achieves significant advantages over DQN, DDQN, Duel DQN, and human baseline across 57 Atari games. QR-DQN\cite{dabney2018distributional} applies regression theory\cite{koenker2005quantile} to distributional RL and achieves performance beyond C51. Specific distortion risk functions based on human decision-making habits are proven effective by IQN\cite{dabney2018implicit} in most Atari games. 
Furthermore, DFAC \cite{pmlr-v139-sun21c} extends the value decomposition method to the field of distributional RL in MAS, and develops the joint quantile function decomposition to extend  VDN and QMIX to DDN and DMIX on condition of stable and noise-free rewards. However, a noticeable drawback of distributional RL is the increased computation and interaction cost to train quantile networks alongside the TD update.
For combatting the computation cost, DRE-MARL \cite{hu2022distributional} designs the multi-action-branch reward estimation and policy-weighted reward aggregation for stabilized training. Additionally, distortion risk functions, which primarily adjust the risk preference of policy in the action value distribution and are mainly employed in tasks where the correlation between risk and reward is evident, are proposed by \cite{dabney2018implicit}. Despite the progress in single-agent setting, the adoption of distortion risk functions is worthy of further investigation, as it may disrupt the collaboration among MAS in MARL, and convergence issues may arise if theoretical guarantees lack \cite{dabney2018implicit, pmlr-v139-sun21c}. 
Meanwhile, the interaction cost can be compensated by generative models \cite{LegProstheses, DA_DGM}.
Among them, DM\cite{DDPM} is renowned for its astonishing capability in the fields of image synthesis \cite{DMvsGAN}, image inpainting, color restoration, image decompression\cite{Palette} and text-to-image\cite{GLIDE, DALLE2}. RL has reaped the benefits of DMs with significant improvement in cross-task learning and experience augmentation. For example, \cite{DPRL} and \cite{MTaskRL} use DM for policy generation in single-agent tasks.  J. Geng \cite{diffusion_policies_marl} leverages DM to encapsulate policies within the MAS context, thereby fostering efficient and expressive inter-agent coordination.
DOM2 \cite{li2023beyond} incorporates DM into the policy network and proposes a trajectory-based data-augmentation scheme in training. However, mostly in single-agent settings, these studies lack validation in non-stationary and noisy environments. Moreover, the approximation errors of DM in generating samples also lack theoretical analyses.

\section{Background and Problem Formulation}\label{sectionProblem}

This section mainly discusses the base model in our work and the formulations of relevant methods.
The main notations used in this paper are listed in Table \ref{mainnotations}. 
\begin{table}[]
    \caption{Main Notations Used in This Paper.}
    \label{mainnotations}
    \centering
    \begin{tabular}{l m{4.5cm}}
        \hline
        Notation                                               & Definition                                             \\
        \hline
        $N$                                                    & Number of agents                                       \\
        $\mathcal{I}$                                          & Set of agents                                          \\
        $ \mathcal{S}$, $\bm{\mathcal{A}}$, and $ \boldsymbol{\mathcal{O}} $ & The state, joint action and joint observation space          \\
        $s$                                                    & Global state                                           \\
        $ \bm a = [a^{[1]},\cdots,a^{[N]}]$                    & Joint action of all agents                             \\
        $ \bm o = [o^{[1]},\cdots,o^{[N]}]$                    & Joint observation of all agents                        \\
        $ Q $ and $ V $                                        & Action value and state value function                     \\
        $ \bm w = [w^{[1]},\cdots,w^{[N]}]^\top$               & Joint weight vector of all agents                      \\
        $ {\mu}_i, {\sigma}_i $                                & Mean and standard deviation related to agent $i$       \\
        $\bm{\pi} = [{\pi}^{[1]},\dots,{\pi}^{[N]}]$           & Joint policy of all agents                             \\
        $\bm \theta = [\theta^{[1]},\cdots, \theta^{[N]}]$                    & Neural network parameters of all agents                \\
        $ Z $                                                  & Random variable fitting action value distribution         \\
        $F_{Z}$                                                & CDF of random variable $Z$                               \\
        $\tau$                                                 & Quantile value of a target distribution                \\
        $\rho (\tau )$                                         & Distortion risk function about quantile $\tau$         \\
        $\gamma$                                               & Discount factor                                        \\
        $ R $                                                  & Random variable fitting reward distribution            \\
        $ r $ and $ \mathcal{R} $                              & Reward sample produced by the function $ \mathcal{R} $ \\
        $ r^{<k>} $                                            & Generated reward in $k$-th step of DM                  \\
        $ K $                                                  & Total diffusion steps of DM                            \\
        $\upsilon_1,\cdots,\upsilon_K$ and $\omega_1,\cdots,\omega_K$  & Predefined hyper-parameters in $K$ steps of DM \\
        $\Theta$                                               & Parameters of DM                                       \\
        $\mathcal{D}_q(\cdot)$                                 & Distribution of random variable or its sample              \\
        % $ G $                                                  & Discounted accumulated global reward                   \\
        $d_p$                                                  & Wasserstein distance computed by $L_p$-norm            \\
        $\mathcal{D}$                                          & Global reward distribution                             \\
        $\hat{\mathcal{D}}_l$                                  & Decomposed distribution of local reward               \\
        \hline
    \end{tabular}
    \vspace{-0.5cm}
\end{table}
\subsection{System Model}\label{total_Preliminaries}

A cooperative multi-agent task can be described as the decentralized partially observable Markov decision process (Dec-POMDP), consisting of a tuple $ \langle \mathcal S, \bm{\mathcal{A}}, \mathcal{P}, \mathcal{R}, 
\boldsymbol{\mathcal{O}}, \Omega, \mathcal{I}, \gamma \rangle$ \cite{oliehoek2016concise, rashid2020monotonic}. Specifically, $s\in \mathcal S$ denotes the global state and $\bm{a} = [a^{[1]},\cdots,a^{[N]}] \in \bm{\mathcal{A}}$ is the joint action of $N$ agents. The joint action causes a transition from current state $s$ to next state $s'$ according to the transition function $\mathcal{P}(s'|s, \bm a): \mathcal{S} \times \bm{\mathcal{A}} \times \mathcal{S} \rightarrow [0, 1]$. Specially, all agents share a global reward function $\mathcal{R}: \mathcal{S} \times \bm{\mathcal{A}} \rightarrow \mathbb{R}$, which means its output $R(s, \bm a)$ follows a global reward distribution $\mathcal{D}$ when given $s$ and $\bm a$, and $ r \in \mathbb{R}$ is sampled from $R(s, \bm a)$. Each agent $i \in \mathcal{I} \equiv \{1,\cdots, N\}$ draws individual observation $o^{[i]} \in 
 \mathcal{O}^{[i]}$ where $ \mathcal{O}^{[i]} \in \bm{\mathcal{O}}$ according to the observation function $\Omega(s, i): \mathcal S \times \mathcal{I} \rightarrow  \mathcal{O}^{[i]} $. 
 Generally, the joint observation $\bm{o}=[o^{[1]}, \cdots, o^{[N]}]$ and we assume that $s$ and $\bm o$ are equivalent.
 Besides, $\gamma\in [0,1)$ is a discount factor. Within the framework of Dec-POMDP in $t$-th step, a joint policy $\bm \pi = [{\pi}^{[1]},\dots,{\pi}^{[N]}]$  corresponds to a joint action value function as 
\begin{equation}
\label{eq:Qvalue}
Q^{\bm \pi}(s_t, \bm{a}_t) = \mathbb{E}_{s_{t:\infty}, \bm{a}_{t:\infty}}[Z^{\bm \pi}(s_t, \bm{a}_t)], 
\end{equation}
where the random variable for the discounted reward summation $Z^{\bm \pi}(s_t, \bm{a}_t) =\sum_{m=0}^{\infty} \gamma^m R(s_{t+m}, \bm a_{t+m})$ has a corresponding CDF $F_Z$. Notably, when the quantile $\tau$ of the CDF is further distorted by distortion risk function $\rho(\tau): [0, 1] \rightarrow [0, 1]$, the distorted action value function $Q_{\rho}({o}, {a})$ and distorted random variable $Z_{\rho}({o}, {a})$ can be expressed in terms  of $\rho(\tau)$ as
\begin{equation}
\label{rho_Q_E}
    Q^{\pi}_{\rho}({o}, {a})
    =\mathbb{E}\big[Z^{\pi}_{\rho}(o, a)\big]
    =\int z \text d\rho\big[F_{Z}(z)\big].
\end{equation}
Additionally, DM gradually adds Gaussian noises to sample $r$ from $\mathcal{D}$ by $K$ times and then reconstructs it in reverse. In the diffusion process, $r^{<k>} \sim \mathcal{D}_q(r^{<k>})$ is a generated sample in $k$-th step, $k \in [0, \cdots, K]$, where $\mathcal{D}_q(\cdot)$ represents the distribution of the corresponding sample. More information regarding distortion risk functions and DM shall be given in Appendix \ref{drf_exams} and Appendix \ref{dm_analysis}, respectively.

\subsection{Value Decomposition} \label{sec:valuedecom}   
Value decomposition methods \cite{sunehag2017value,rashid2020monotonic,rashid2020weighted,son2019qtran} belong to one of the typical methods to learn decentralized action value networks from globally shared rewards. As one of the pioneering works of value decomposition, QMIX \cite{rashid2020monotonic} holds that under a monotonicity constraint, the joint of individually optimal actions is equivalent to the optimal joint action regarding the action value function, which is also a crucial criterion for MARL with distributed deployment of agents. That is,
\begin{equation}\label{q_mix_org_eq}
\mathop{\text{argmax}}\limits_{\bm{a}} Q_{\text{global}}(\bm{o}, \bm{a})=
\begin{pmatrix}
\mathop{\text{argmax}}\limits_{a^{[1]}}Q_{1}(o^{[1]}, a^{[1]}) \\
\vdots \\
\mathop{\text{argmax}}\limits_{a^{[N]}}Q_{N}(o^{[N]}, a^{[N]})
\end{pmatrix},
\end{equation}
where $Q_{\text{global}}(\bm{o}, \bm{a})$ indicates the joint action value function of all agents and $Q_i(o^{[i]},a^{[i]})$ is the action value function corresponding to agent $i$. 
Besides, the monotonicity constraint can be formally written as 
\begin{equation}
\label{eq:monotonicity}
\frac{\partial Q_{\text{global}}(\bm{o}, \bm{a})}{\partial Q_i(o^{[i]},a^{[i]})} \geq 0,\, \forall i=1,2,\cdots,N.
\end{equation}
Furthermore, in order to satisfy \eqref{eq:monotonicity}, a mixing network is proposed to combine ${Q_i}(o^{[i]},a^{[i]})$ into $Q_{\text{global}}(\bm o, \bm a)$ with non-linear transformation.

\subsection{Distributional RL}\label{sec:drl}
Distributional RL \cite{bellemare2017distributional} aims to learn the distribution of action or state values, so as to capture more comprehensive information about random rewards (i.e., noisy rewards that follow a specific distribution). Besides, distributional RL benefits agents to hold different risk tendencies (e.g., risk-averse and risk-seeking) to choose preferred policies  \cite{dabney2018implicit}. 
Specially for multi-agent distributional RL, given the state $s$ and joint action $\boldsymbol{a}$, 
 the transition operator $P^{\pi}$ is defined as
\begin{equation}
\label{transition_operator}
    P^{\pi}Z^{\pi}(s,\boldsymbol{a}) \overset{D}{:=} Z^{\pi}(S', \boldsymbol{A}'), 
\end{equation}
where the operator $U \overset{D}{:=} V$ indicates that the random variable $U$ is distributed according to the same law as $V$. Consistent with \eqref{eq:Qvalue}, $Z^{\pi}(s,\boldsymbol{a})$ denotes a random variable with the expectation $Q^{\pi}(s,\boldsymbol{a})$, and capital letters $S'$ and $\boldsymbol{A}'$ are used to emphasize the random nature of next state-action pair (i.e., $S' \sim \mathcal{P}(\cdot|s,\boldsymbol{a})$ and $\boldsymbol{A}' \sim \boldsymbol{\pi}(\cdot|S')$). On this basis, the distributional Bellman operator $\mathcal{T} ^{\pi}$ \cite{bellemare2017distributional} of MAS is defined as:
\begin{equation}
\label{Bellman_operator}
    \mathcal{T} ^{\pi}Z^{\pi}(s,\boldsymbol{a}) \overset{D}{:=} R(s,\boldsymbol{a}) + \gamma P^{\pi}Z^{\pi}(s,\boldsymbol{a}).
\end{equation}
That is to say, $Z^{\pi}(s,\boldsymbol{a})$ is characterized by the interaction of three random variables (i.e., the reward $R(s,\boldsymbol{a})$, the next pair of state-action $(S^{\prime},\boldsymbol{A}^{\prime})$ and the corresponding $Z^{\pi}(S^{\prime},\boldsymbol{A}^{\prime})$ according to \eqref{transition_operator}). 
In distributional RL, \eqref{Bellman_operator} is used to update $Z^{\pi}(s,\boldsymbol{a})$ with the next state and action by temporal difference.

Accordingly, on account of the definition of $Z^{\bm \pi}(s_t, \bm{a}_t)$ in \eqref{eq:Qvalue} and the equivalence between $s_t$ and $\bm{o}_t$, the random variable for the global discounted reward summation can be rewritten as
\begin{equation}
    Z^{\boldsymbol{\pi}}_{\text{global},t}(\bm{o}_t, \bm{a}_t) = \sum\nolimits_{m=0}^{\infty}\gamma^m R(\bm{o}_{t+m}, \boldsymbol{a}_{t+m}),\label{z_defination}
\end{equation}
% where $\bm{\pi}= [\pi^{[1]},\cdots,\pi^{[N]}]$ denotes the joint policy of all agents. 
% As mentioned in Section \ref{task_introduction}, 
$R(\boldsymbol{o}_{t+m}, \boldsymbol{a}_{t+m})$ is a random variable of the noisy reward following the distribution $\mathcal{D}$, which can be approximately regarded as a GMM and decomposed into $N$ Gaussian components. 

On the other hand, it is critical to measure the distributional difference between the true distribution of noisy reward and the approximated one. In this regard, Wasserstein metric \cite{dabney2018distributional} can be leveraged and mathematically defined as
\begin{equation}
\label{Wasserstein}
d_p(F,G)=\big(\int_0^1 |F^{-1}(u)-G^{-1}(u)|^p\text du\big)^{\frac{1}{p}},
\end{equation}
where $F$ and $G$ are the CDFs of random variables $U$ and $V$, respectively. $p < \infty$ implies to take an $L_p$ norm for this 
metric. Furthermore, in single-agent settings, some robust approaches have been proposed by using quantiles to evaluate and then fit the action value distribution \cite{dabney2018distributional,dabney2018implicit}.

% \subsection{Distortion Risk Function}
% \label{sec:distortion_risk}

% Let $F_{Z}$ , $F^{-1}_{Z}$ be the CDF of $Z$ and its inverse function respectively. 
% Let $F^{-1}_{Z}$ be the inverse function of $Z$'s CDF.

% \subsection{Diffusion Model}
% \label{dm_intro}

\subsection{Problem Formulation}\label{task_introduction}
\begin{figure*}
\centerline{\includegraphics[width=.95\textwidth, height=0.45\textheight]{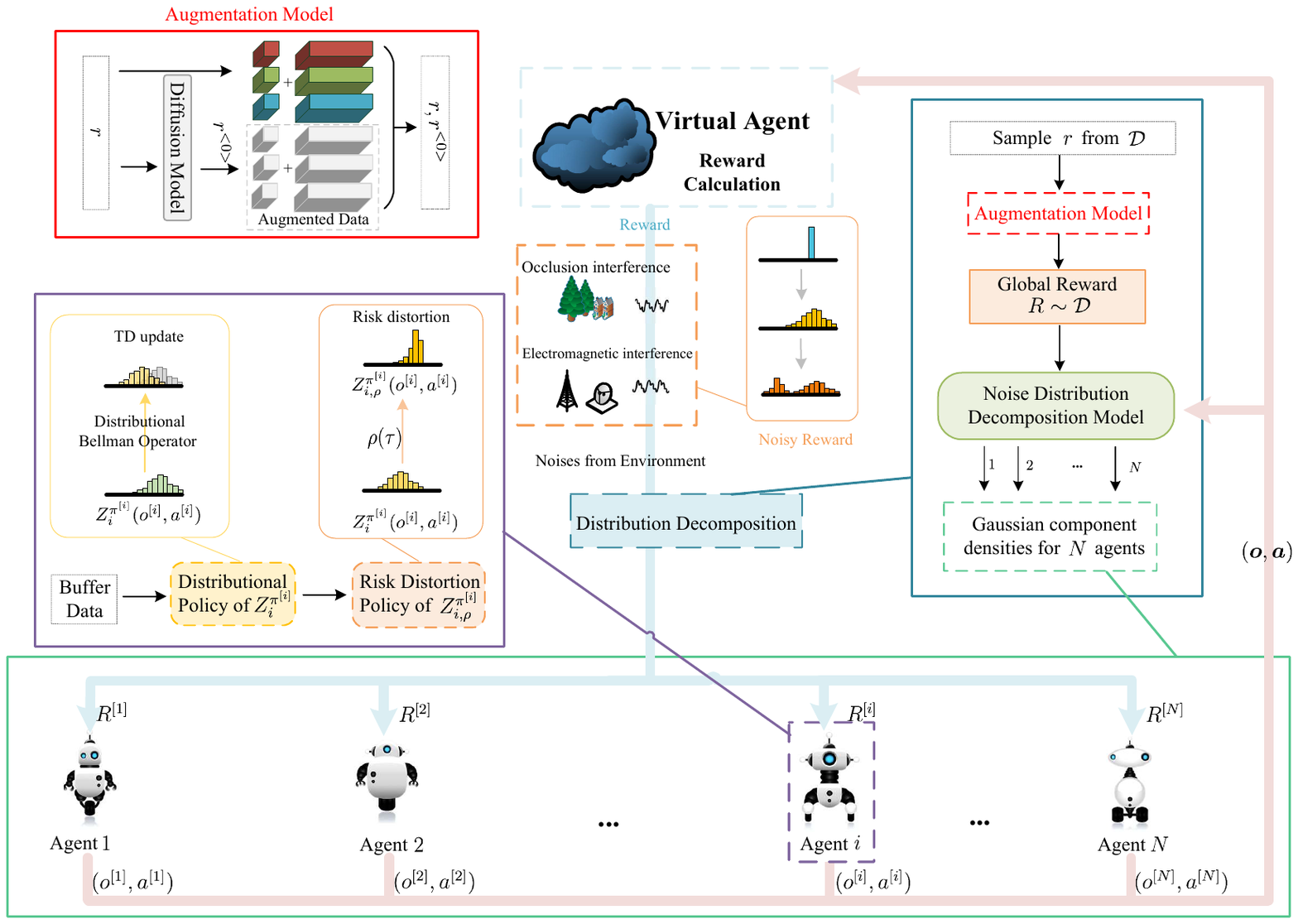}}%height=0.425\textheight
\caption{The scenario and framework of NDD-based cooperative multi-agent distributional RL with noisy rewards.}
\label{baseproces}
\label{task_fig}
\end{figure*}

In this paper, we primarily consider $N$ RL agents (i.e., a group of RL-empowered distributed robots equipped with intelligent sensing devices and basic computing power) to perform certain cooperative tasks as shown in Fig. \ref{baseproces}. Consistent with the general settings in Dec-POMDP, at each time-step $t$, agent $i$ capably utilizes on-board sensors to access a partial environmental observation $o^{[i]}_t \in  \mathcal{O}^{[i]}$. Observations may include partial states of some teammates, received messages, or the relative positions to certain target points. This often depends on the specific task executed by the whole team.
With a DM for data augmentation, every agent attempts to learn action value distributions and remaps them with distortion risk functions for yielding policy corresponding to different tasks. Based on the observation $o^{[i]}_t$, agent $i$ responds with action $a^{[i]}_t$ (i.e., movement direction and message exchange) according to its policy $\pi^{[i]}(a^{[i]}_t|o^{[i]}_t)$. Afterwards, the environmental state transits from $s_t$ to next state $s_{t+1}$ according to the transition function $\mathcal{P}(s_{t+1}|s_t, \bm a_t)$. Meanwhile, other observations and global information can be used for the calculation of reward from $\mathcal{R}$ by a suitable ``virtual'' agent (e.g., a centralized server or a computing-capable global-oriented robot). Notably, since the observation provided by sensors, or exactly the reward, could be polluted by internal or external noise, we formulate the reward by the random variable $R \sim \mathcal{D}$, consistent with Section \ref{total_Preliminaries}. More details are in Fig. \ref{baseproces}.
In the subsequent sections, \blus{targeted at the aforementioned MARL settings, we mainly discuss how to apply our method NDD to tackle the noisy rewards, so as to maximize the average reward of episodes for the entire team.}

\section{Noise Distribution Decomposition for MARL}\label{sectionNDD}

In this section, particularly focusing on general MARL settings, we provide details of NDD toward tackling the globally shared noisy rewards with improved learning accuracy and stability. Specifically, following the idea of distributional RL, we choose the action value distribution (i.e., the distribution that $Z^{\boldsymbol{\pi}}_{\text{global}}(\bm{o}, \bm{a})$ in \eqref{z_defination} follows) rather than the action value function to update agents' joint policy $\boldsymbol{\pi}$. As mentioned earlier, it will also be intractable to directly decompose the distribution of the global $Z^{\boldsymbol{\pi}}_{\text{global}}(\bm{o}, \bm{a})$. Thus, we further approximate the distribution of the global reward as a GMM \cite{reynolds2009gaussian} and propose to decompose the noisy reward distribution into individual parts, so as to facilitate the update of local distributions. Meanwhile, to ensure the consistency between local and global action values, we also provide theoretical details of the distribution decomposition and present the alternative functions for distortion risk measures. Finally, DM is introduced to enhance buffer data, and the corresponding generation and approximation errors are analyzed.

\subsection{GMM-based Distributional MARL}
As a parametric probability density function (PDF), a GMM represents a weighted sum of Gaussian component densities \cite{reynolds2009gaussian}. Mathematically, taking the example of any scalar $x$ from random variable $X \sim \texttt{\large GMM}(\mu_1,\cdots,\mu_N,\sigma^2_1,\cdots,\sigma^2_N,w^{[1]},\cdots,w^{[N]})$,
the corresponding PDF can be expressed as

\begin{equation}
\begin{aligned}\label{gmm_eq}
&f(x;\mu_1,\cdots,\mu_N,\sigma^2_1,\cdots,\sigma^2_N,w^{[1]},\cdots,w^{[N]})\\
=&\sum\nolimits_{i=1}^{N}w^{[i]}\ g(x;\mu_i, \sigma^2_i),
\end{aligned}
\end{equation}
where $w^{[1]},\dots,w^{[N]}$ are mixture weights with $\sum_{i=1}^Nw^{[i]}=1$. $g(x;\mu_i, \sigma^2_i)$ denotes the PDF of the Gaussian distribution with mean $\mu_i$ and variance $\sigma_i^2$. According to Wiener’s Tauberian theorem\cite{wiener1932tauberian, narayanan2009wiener}, GMM, which has density completeness, can essentially fit the shape of arbitrary distribution with absolutely integrable PDF to describe data from single or multiple noise sources \cite{NoisycGMM}, so $\mathcal{D}$ can also be approximated as $\hat{\mathcal{D}}$ which is a GMM. 
Thus, $\hat{\mathcal{D}}$ can be further decomposed into Gaussian components $\hat{\mathcal{D}}_l^{[1]}(\theta^{[1]}),\cdots,\hat{\mathcal{D}}_l^{[N]}(\theta^{[N]})$ corresponding to different individual agents. 

Without loss of generality, let $R_l^{[1]},\dots,R_l^{[N]}$ be independent decomposed local rewards where $R_l^{[i]}\sim \hat{\mathcal{D}}_l^{[i]}(\theta^{[i]})$. We further define a function $\Psi(R_l^{[1]},\dots,R_l^{[N]};w^{[1]},\dots,w^{[N]})$ to statistically combine 
 $R_l^{[i]}$ with an assigned probability of $w^{[i]}$, and thus $\Psi(R_l^{[1]},\dots,R_l^{[N]};w^{[1]},\dots,w^{[N]})$ conforms to a GMM with its expectation equaling to $\sum_{i=1}^N w^{[i]}\mathbb{E}(R_l^{[i]})$. In other words, it satisfies the additive property.
Hence, the global reward $ R(\bm{o}_{t+m}, \bm{a}_{t+m})$ in \eqref{z_defination} approximately equals  $ \Psi(R_l^{[1]},\dots,R_l^{[N]};w^{[1]}_m,\dots,w^{[N]}_m)$, and \eqref{z_defination} can be rewritten as \eqref{eq:global_z} on Page \pageref{eq:global_z},
\begin{figure*}
    \begin{equation} \label{eq:global_z}
    Z^{\bm{\pi}}_{\text{global},t}(\bm{o}_t, \bm{a}_t) 
    =  \sum\nolimits_{m=0}^{\infty}\gamma^m \Psi\big[R_l^{[1]}(o^{[1]}_{t+m},a^{[1]}_{t+m}),\dots,R_l^{[N]}(o^{[N]}_{t+m},a^{[N]}_{t+m}); w^{[1]}_m,\dots,w^{[N]}_m\big].
    \end{equation}
    \hrulefill
\end{figure*}
with $ \mathbb{E}[Z_{\text{global},t}^{\bm \pi}(\bm{o}_t, \bm{a}_t)] = Q_{\text{global},t}^{\bm \pi}(\bm{o}_t, \bm{a}_t)$, which is action value of the whole MAS in $t$-th step given joint observation, action and policy $\bm{o}, \bm{a}, \bm{\pi}$.
Furthermore, the local $Z^{\pi^{[i]}}_{i,t}(o^{[i]}_t, a^{[i]}_t)$ of agent $i$ can be represented as
\begin{equation}
Z^{\pi^{[i]}}_{i,t}(o^{[i]}_t, a^{[i]}_t) =  \sum\nolimits_{m=0}^{\infty}\gamma^m R_l^{[i]}(o^{[i]}_{t+m},a^{[i]}_{t+m}). \label{eq:local_z}
\end{equation}
For simplicity of representation, we might omit $ \bm{o}_t,\bm{a}_t,\bm{\pi}, t$ in $ Z^{\bm{\pi}}_{\text{global},t}(\bm{o}_t, \bm{a}_t)$ and $ o^{[i]}_t,a^{[i]}_t,\pi^{[i]},t$ in $ Z^{{\pi}^{[i]}}_{i,t}(o^{[i]}_t,a^{[i]}_t)$, and only use the notations $Z_{\text{global}}$ and $Z_{i}$.

\subsection{Consistency of Global and Local Optimal Actions}\label{sec:Consistency}
\iffalse
Generally speaking, the optimal policy requires choosing actions while satisfying
\begin{equation}
\pi^*(s) = \argmax_a\mathbb{E}[Z(s,a)]
\end{equation}
in distributional RL.
\fi

First, we show that the monotonicity constraint mentioned in Section \ref{sec:valuedecom} holds for the GMM-based distributional MARL on some conditions.
\begin{theorem}
\label{thm:monotonicity}
For any $i=1,\dots,N$ and $Z_{i}$ defined by \eqref{eq:local_z} in any $m$-th step, if $w^{[i]}_m\geq 0$ we have
\begin{equation}
\frac{\partial{\mathbb{E}(Z_{\text{global}})}}{\partial{\mathbb{E}(Z_{i})}} \geq 0\label{original_condition}.
\end{equation}
\end{theorem}
\begin{proof}
As the definition in \eqref{eq:global_z}, the expectation of $Z_{\text{global}}$ computed from $M$ time-steps can be expressed as
\begin{align}\label{local_dist_decompose_to_matrice}
\mathbb{E}\big[Z_{\text{global}}\big]  = \text{tr}(\bm{W}^\top \boldsymbol{R}^{\#}), 
\end{align}
where \begin{equation}
    \boldsymbol{W} = 
\begin{bmatrix}
w^{[1]}_0 & \hdots & w^{[N]}_0\\
\vdots & \vdots & \vdots \\
w^{[1]}_M & \hdots & w^{[N]}_M
\end{bmatrix},
\end{equation} 
\begin{equation}
\boldsymbol{R}^{\#} =  \begin{bmatrix}
\gamma^0 \mathbb{E}(R^{[1]}_{l,0}) & \hdots & \gamma^0 \mathbb{E}(R^{[N]}_{l,0}) \\
\vdots & \vdots & \vdots\\
\gamma^{M} \mathbb{E}(R^{[1]}_{l,M}) & \hdots & \gamma^{M} \mathbb{E}(R^{[N]}_{l,M})
\end{bmatrix},
\end{equation}
and $R^{[i]}_{l,t}$ is the abbreviation of $R_l^{[i]}(o^{[i]}_t, a^{[i]}_t)$. Similarly, the local $Z_{i}$ can be formulated as
\begin{equation}
\begin{split}
\begin{bmatrix}
\mathbb{E}(Z_{1}) \hdots \mathbb{E}(Z_{N})
\end{bmatrix}
= \boldsymbol{1}^\top_{(M+1)\times 1}\boldsymbol{R}^{\#}.
\end{split}
\end{equation}

Taking $\kappa = \mathop{\text{min}}(w^{[i]}_m)$, $\sum_i w^{[i]}_m=1$ ($\forall t$) implies $0 \leq \kappa \leq \frac{1}{N}$. Therefore, \eqref{local_dist_decompose_to_matrice} can be rewritten as
\begin{align}
& \mathbb{E}\big[Z_{\text{global}}\big] 
=  \text{tr}(\boldsymbol{W}^\top \boldsymbol{R}^{\#}) \nonumber\\
\geq & \kappa \cdot \text{tr}(\boldsymbol{1}^\top _{(M+1)\times N} \boldsymbol{R}^{\#})\label{eq:decomposition_relationship}
=  \kappa \cdot \boldsymbol{1}^\top_{(M+1)\times 1}\boldsymbol{R}^{\#} \boldsymbol{1}_{N\times 1}\\
= & \kappa \cdot 
\begin{bmatrix}
\mathbb{E}(Z_{1}) & \hdots & \mathbb{E}(Z_{N})
\end{bmatrix}
\boldsymbol{1}_{N\times 1}. \nonumber
\end{align}
Consequently, \eqref{original_condition} is satisfied as $\frac{\partial{\mathbb{E}(Z_{\text{global}})}}{\partial{\mathbb{E}(Z_{i})}} \geq \kappa \geq 0$.
\end{proof}
Recalling that $\mathbb{E}[Z_{\text{global}}] = Q_{\text{global}}(\bm{o}, \bm{a})$ and $\mathbb{E}[Z_{i}] = Q_i(o^{[i]}, a^{[i]})$, together with \eqref{q_mix_org_eq}, we have following corollary.
\begin{corollary}
\begin{equation}
 \label{q_mix_eq}
\mathop{\text{argmax}}\nolimits_{\bm{a}} \mathbb{E}\big[Z_{\rm{global}}\big]=
\begin{pmatrix}
\mathop{\text{argmax}}\nolimits_{a^{[1]}}\mathbb{E}\big[Z_{1}\big] \\
\vdots \\
\mathop{\text{argmax}}\nolimits_{a^{[N]}}\mathbb{E}\big[Z_{N}\big]
\end{pmatrix}.
\end{equation}
\end{corollary}
\noindent\emph{Remark}: The proof of Theorem \ref{thm:monotonicity} implies that some excessive small weights mean trivial contributions to the group. \blus{Meanwhile, from the perspective of minimizing the reward gap between an ideal policy with full observability and partially observable local policies \cite{return_gap}, a more uniform distribution of weights can decrease this gap between $\mathbb{E}(Z_{\text{global}})$ and its lower bound $\kappa \cdot 
\begin{bmatrix}
\mathbb{E}(Z_{1}) & \hdots & \mathbb{E}(Z_{N})
\end{bmatrix}
\boldsymbol{1}_{N\times 1}.$} 
Especially, for cases where weights are evenly divided among agents (i.e., $w^{[i]}_m = \kappa = \frac{1}{N},\forall i = 1,\cdots, N$), $\mathbb{E}(Z_{\text{global}}) = 
\frac{1}{N} \sum_i \mathbb{E}[Z_{i}]$, which also satisfies \eqref{original_condition}. However, only enforcing such a strong constraint on weights will incur large errors in decomposition results. Thus, it is meaningful to carefully calibrate the design of the loss function by appending additional constraint terms. 

\subsection{Design of Distortion Risk Function}
As mentioned in \cite{dabney2018implicit}, the distortion risk function $\rho(\tau)$ enhances the applicability and flexibility of distributional RL. Based on $\rho(\tau)$, Dabney \textit{et al.} \cite{dabney2018implicit} introduce an parameter $\eta$ to make $\rho_{\eta}(\tau)$ more flexible and adaptable. However, for MARL, $\rho_{\eta}(\tau)$ needs to be subjected to certain constraints in order to ensure that the distorted action value distribution also satisfies the monotonicity constraint in Theorem \ref{thm:monotonicity}.

\begin{theorem}
    \label{thm:remap_consistency}
    Given $\frac{\partial{\mathbb{E}(Z_{\text{global}})}}{\partial{\mathbb{E}(Z_{i})}} \geq 0$, if $\rho^{\prime}_{\eta}(\tau)\in (0, +\infty)(\forall \tau)$ we have
    \begin{equation} \label{remap_consistency}
        \frac{\partial{\mathbb{E}(Z_{\text{global},{\rho}})}}{\partial{\mathbb{E}(Z_{i,{\rho}})}} \geq 0.
    \end{equation}
    Here $Z_{\text{global},{\rho}}$ and $Z_{i,{\rho}}$ represent that $Z_{\text{global}}, Z_{i}$ have undergone the remap distortion $\rho_{\eta}(\tau)$.
\end{theorem}
\begin{proof}
    According to \eqref{rho_Q_E},
    \begin{equation}
    \begin{aligned} \label{remap_consistency_00}
        \frac{\partial{\mathbb{E}(Z_{\text{global},{\rho}})}}{\partial{\mathbb{E}(Z_{i,{\rho}})}}
        & = \frac{\partial{\Big\{\int z \text d \rho_{\eta} \big[F_{\text{global}}(z)\big] \Big\}}}{\partial{\Big\{\int z \text d \rho_{\eta} \big[F_{i}(z)\big] \Big\}}}\\
        & = \frac{\partial{\big[\int z \rho^{\prime}_{\eta}(\tau_{\text{global}}) \text d F_{\text{global}}(z) \big]}}{\partial{\big[\int z \rho^{\prime}_{\eta}(\tau_i) \text d F_{i}(z) \big]}},
    \end{aligned}
    \end{equation}
    where $F_{\text{global}}, F_{i}$ denote the CDFs of $Z_{\text{global}}$, $Z_{i}$; and $\tau_{\text{global}}, \tau_i$ are their quantiles, respectively. Based on the positive-bound assumption of $\rho^{\prime}_{\eta}(\tau)(\forall \tau)$, let $\rho^{\prime}_{\eta}(\tau) \in [\rho^{\prime}_{\text{min}},\rho^{\prime}_{\text{max}}]$, \eqref{remap_consistency_00} can be rewritten as
    \begin{align}
        \frac{\partial{\mathbb{E}(Z_{\text{global},{\rho}})}}{\partial{\mathbb{E}(Z_{i,{\rho}})}}
        & \geq \frac{\rho^{\prime}_{\text{min}}}{\rho^{\prime}_{\text{max}}} \frac{\partial{\big[\int z \text d F_{\text{global}}(z) \big]}}{\partial{\big[\int z  \text d F_{i}(z) \big]}}\\ \nonumber
        & = \frac{\rho^{\prime}_{\text{min}}}{\rho^{\prime}_{\text{max}}} \frac{\partial{\mathbb{E}(Z_{\text{global}})}}{\partial{\mathbb{E}(Z_{i})}}.
    \end{align}
    \eqref{remap_consistency} holds obviously when $\rho^{\prime}_{\eta}(\tau)\in (0, +\infty)$.
\end{proof}

\noindent\emph{Remark}: Intuitively, the distortion risk functions including CPW\cite{CPW, CPW071}, WANG\cite{WANG}, POW\cite{dabney2018implicit}, and CVaR\cite{CVaR} given in Appendix \ref{drf_exams} satisfy the condition (i.e., $\rho^{\prime}_{\eta}(\tau)\in (0, +\infty)$), when they take the values in Table \ref{riskfunnotations} as \cite{dabney2018implicit}. Furthermore,
\begin{equation}
 \label{q_mix_eq_rho}
\mathop{\text{argmax}}\nolimits_{\bm{a}} \mathbb{E}\big(Z_{\text{global},\rho}\big)=
\begin{pmatrix}
\mathop{\text{argmax}}\nolimits_{a^{[1]}}\mathbb{E}\big(Z_{1, \rho}\big) \\
\vdots \\
\mathop{\text{argmax}}\nolimits_{a^{[N]}}\mathbb{E}\big(Z_{N, \rho}\big)
\end{pmatrix}
\end{equation}
holds. Thus, it lays the very foundation to adopt these distortion risk measures in our case. 

In summary, based on Theorem \ref{thm:monotonicity} and Theorem \ref{thm:remap_consistency}, the monotonicity of the reward distribution decomposition is theoretically validated under nonnegative $w^{[i]} (i=1,\cdots,N)$ and increasing $\rho_{\eta}(\tau)$. 
\blut{From this monotonicity, the collective performance of MAS will not deteriorate when every agent improves its performance, with consistent global and local convergences. }
% \blut{From this monotonicity, the collective performance of MAS will not deteriorate when every agent individually updates its policy and convergence in MAS can be guaranteed as long as the local policy is effective.}

\subsection{DM-based Data Augment}
The action value distribution reflects more risk information than the action value function, but the interaction cost increases between agents and environment \cite{bellemare2017distributional}. Especially in MAS, additional interaction demands several times greater than that of a single agent. Hence, using a generative model for data augmentation becomes a feasible compromise \cite{LegProstheses, DA_DGM}. Besides, in order to maintain good compatibility with GMM, DM \cite{DDPM} is used in the NDD, as Theorem \ref{them_generation_error} and Theorem \ref{them_approximation_error} demonstrate bounded generation error and approximation error, which measure the distance between the generated sample and the original sample or the GMM, respectively.
\begin{theorem}[Bounded Generation Error]
    \label{them_generation_error}
    Under the assumption that $K = 25$ and $[\omega_1, \cdots, \omega_K]= 0.499\times10^{-2} \times \frac{1}{1+e^{-\boldsymbol{h}}}+10^{-5}$, where $\boldsymbol{h}$ is an array with a length of $K$ evenly spaced between $-6$ and $6$, the expectation of generation error $\mathbb{E}[(r^{<0>} - r)^2]$  satisfies 
    \begin{equation}
        \mathbb{E}\big[(r^{<0>} - r)^2\big] \lesssim \big(\mathbb{E}(r)\big)^2 + 5.1359\mathbb{D}(r).
    \end{equation}
\end{theorem}
\begin{theorem}[Bounded Approximation Error]
    \label{them_approximation_error}
    Under the assumption consistent with Theorem \ref{them_generation_error}, the upper bound of approximation error expectation for $r^{<0>}$ is less than that for $r$, where the approximation error $\xi$ for $r$ is defined as the gap between $r$ and the GMM (i.e., $r=r_{\texttt{GMM}} + \xi$) with $r_{\texttt{GMM}} \sim \texttt{\large GMM}(\cdot)$.
\end{theorem}
\noindent We leave the proofs in Appendix \ref{generation_error} and Appendix \ref{approximation_error}. Moreover, Theorem \ref{them_generation_error} and Theorem \ref{them_approximation_error} show that the gradual inclusion of Gaussian components aligns well with GMM used in the NDD, making the integration of DM with NDD a prudent approach, and further imply how to select appropriate hyperparameters (e.g. diffusion steps $K$).

Next, we elaborate on the combination between DM and NDD. Firstly, as shown in Appendix \ref{dm_analysis}, DM is employed to learn the distribution $\mathcal{D}$ of reward $R$ via its sample $r$ and generate $r^{<0>}$. Afterwards, NDD approximates augmented data to GMM and decomposes it into several Gaussian components for $N$ agents. Each agent updates its policy with distributional RL and certain distortion risk functions. The complete process is visualized in Fig. \ref{task_fig}.

\subsection{Loss Function Design}
\begin{algorithm}[t]
\caption{The \blus{Centralized Training Procedure for} NDD}\label{algorithm}
\KwIn{Agents from $1$ to $N$, the maximum iteration number $T$, \blus{trajectory buffer $\mathcal{B}$}, the error threshold $e_{\min}$ and hyperparameters $\lambda,\alpha$.}
\KwOut{\blus{$\bm \theta = [\theta^{[1]},\cdots, \theta^{[N]}]$ for agents $i= 1, \cdots, N$.}}
Initialize $\bm{\theta}$ and $\bm{w}$ randomly\;
\blus{Initialize $\mathcal{B}$; initialize joint policy $\bm{\pi}=[\pi^{[1]},\cdots,\pi^{[N]}]$\;}
\For{\blums{$iteration = 1$ to $T$}}{
\While{\blus{$\mathcal{B}$ is NOT full}}{
\For{\blums{$i=1$ to $N$}}{
% \blus{Sample $a^{[i]}$ from $\pi^{[i]}$\;}
\blus{$a^{[i]} \sim \pi^{[i]}(\cdot|o^{[i]})$\;}
}
\blus{Execute $\bm a = [a^{[1]},\cdots,a^{[N]}]$ and collect set of trajectories into $\mathcal{B}$\;}
}
Generate $r^{<0>}$ by DM in Algorithm \ref{dm_train_algorithm} when given $r$\;
$r \leftarrow (r,r^{<0>})$, $e\leftarrow +\infty$\;
\While{$ e > e_{\min}$}{
\For{$i = 1$ to $N$}{
Generate $\mu_i(\theta^{[i]})$ and $\sigma_i(\theta^{[i]})$ of $\hat{\mathcal D}_l^{[i]}(\theta^{[i]})$ based on $\langle o^{[i]},a^{[i]} \rangle$ by the local distribution network\;
Calculate the Gaussian PDF ${\hat P^{[i]}}$ of $R_l^{[i]}\sim \hat{\mathcal D}_l^{[i]}(\theta^{[i]})$\;
}
$\boldsymbol{w} \leftarrow \text{Softmax}(\boldsymbol{w}$)\;
Concatenate the generated means as $\bm{X}=[\mu_1(\theta^{[1]}),\dots,\mu_N(\theta^{[N]})]^{\top}$\;
Calculate the PDF of the estimated distribution $\hat{\mathcal{D}}$ in the form of a GMM as ${\hat P}=\sum_{i=1}^N w^{[i]}{\hat P}^{[i]}$ \;
Update $\bm \theta$ and $\boldsymbol{w}$ by minimizing the loss function as in \eqref{eq:total_loss}\;
Calculate $ e = \mathcal{L}_{\text{PDF}}(\bm{\theta}, \bm{w}) $ according to \eqref{eq:PDF_loss}\;
}
% \For{\blums{$i=1$ to $N$}}{
% \blums
% % Compute advantage estimates\;
% % Update local policy $\pi^{[i]}$ by using GAE\;
% % Update critic networks with $R^{[i]}$\;
% Update local policy $\pi^{[i]}$ by using $R^{[i]}$;
% }
\blus{Clear $\mathcal{B}$\;} 
}
\end{algorithm}

As mentioned before, the NDD networks decompose the globally shared reward into weighted individual rewards for cooperative agents by estimating the distributions of decomposed rewards. Therefore, the estimation error, which can be regarded as a gap between the approximated PDF and true PDF of the globally shared reward, shall be minimized. 

Without loss of generality, let $P$ and $\hat P$ be the PDFs of the true reward distribution $\mathcal{D}$ and the estimated one $\hat{\mathcal{D}}$ through the NDD networks respectively. Based on the Wasserstein metric in \eqref{Wasserstein}, the loss function for NDD  as

\begin{equation}
\label{eq:PDF_loss}
\mathcal{L}_{\text{PDF}}(\bm \theta, \bm w) = \int^{r_{\text{max}}}_{r_{\text{min}}} \big[{P}(u)-{\hat P}(u;\bm\theta)\big]^2\text du,
\end{equation}
where $[r_{\text{min}}, r_{\text{max}}]$ denotes the range of the reward distribution. Except the parameters $\Theta$ for DM, $\bm \theta = [\theta^{[1]},\cdots,\theta^{[N]}]$ indicates all parameters of NDD networks corresponding to each agent and $\boldsymbol{w}=[w^{[1]},\dots,w^{[N]}]^\top$. Here both $\bm \theta$ and $\bm w$ are the trainable parameters in the NDD networks. 

On the other hand, there may exist many possible candidates for decomposed Gaussian distributions $\hat{\mathcal{D}}_l^{[i]}(\theta^{[i]})$, $\forall i = 1, \cdots, N$. To alleviate the potential issue of multiple solutions in the decomposition operation, we further impose another penalty term as
\begin{equation}
\mathcal{L}_{\text{mean}}(\bm \theta) = (\bm X-\mu\cdot\boldsymbol{1}_{N\times 1})^\top(\bm X-\mu\cdot\boldsymbol{1}_{N\times 1}),
\end{equation}
where $\bm X=[\mu_1(\theta^{[1]}),\dots,\mu_N(\theta^{[N]})]^\top$ and $\mu = \sum^N_{i=1}\mu_i(\theta^{[i]})$, where $\mu_i(\theta^{[i]})$ denotes the mean of $\hat{\mathcal{D}}_l^{[i]}(\theta^{[i]})$.

For an individual agent, the weight represents its degree of involvement in the decomposition task. Certain excessively small weight indicates that $\mathcal{L}_{\text{PDF}}$ has minimal supervisory impact on that component and may result in insufficient stability of the Gaussian component after decomposition. This becomes particularly evident in some simple noise scenarios, such as when the noise can be represented by the superposition of a few Gaussian components. In such cases, the remaining agents may not receive effective components. With excessively low weights, these agents might exploit extremely deviant reward distributions in training policies, ultimately affecting their performances.
In addition, in order to narrow the disparity between $\mathbb{E}(Z_{\text{global}})$ and its lower bound, as mentioned in Section \ref{sec:Consistency}, weights can be aligned more evenly. Therefore, we define the corresponding penalty term as
\begin{equation}
\mathcal{L}_{\text{weight}}(\boldsymbol{w}) = \Vert \boldsymbol{w} - \frac{1}{N} \cdot \boldsymbol{1}_{N\times 1} \Vert_2 .
\end{equation}

To sum up, given hyperparameters $\lambda$ and $\alpha$, we obtain the Wasserstein metric-based loss function as
\begin{equation}
\label{eq:total_loss}
\mathcal{L}(\bm\theta,\boldsymbol{w}) = \mathcal{L}_{\text{PDF}}(\bm\theta, \bm w) + \lambda \mathcal{L}_{\text{mean}}(\bm \theta) + \alpha \mathcal{L}_{\text{weight}}(\boldsymbol{w}).
\end{equation}
% where $\lambda$ and $\alpha$ are hyperparameters.
\begin{figure}
\centerline{\includegraphics[width=3.0in]{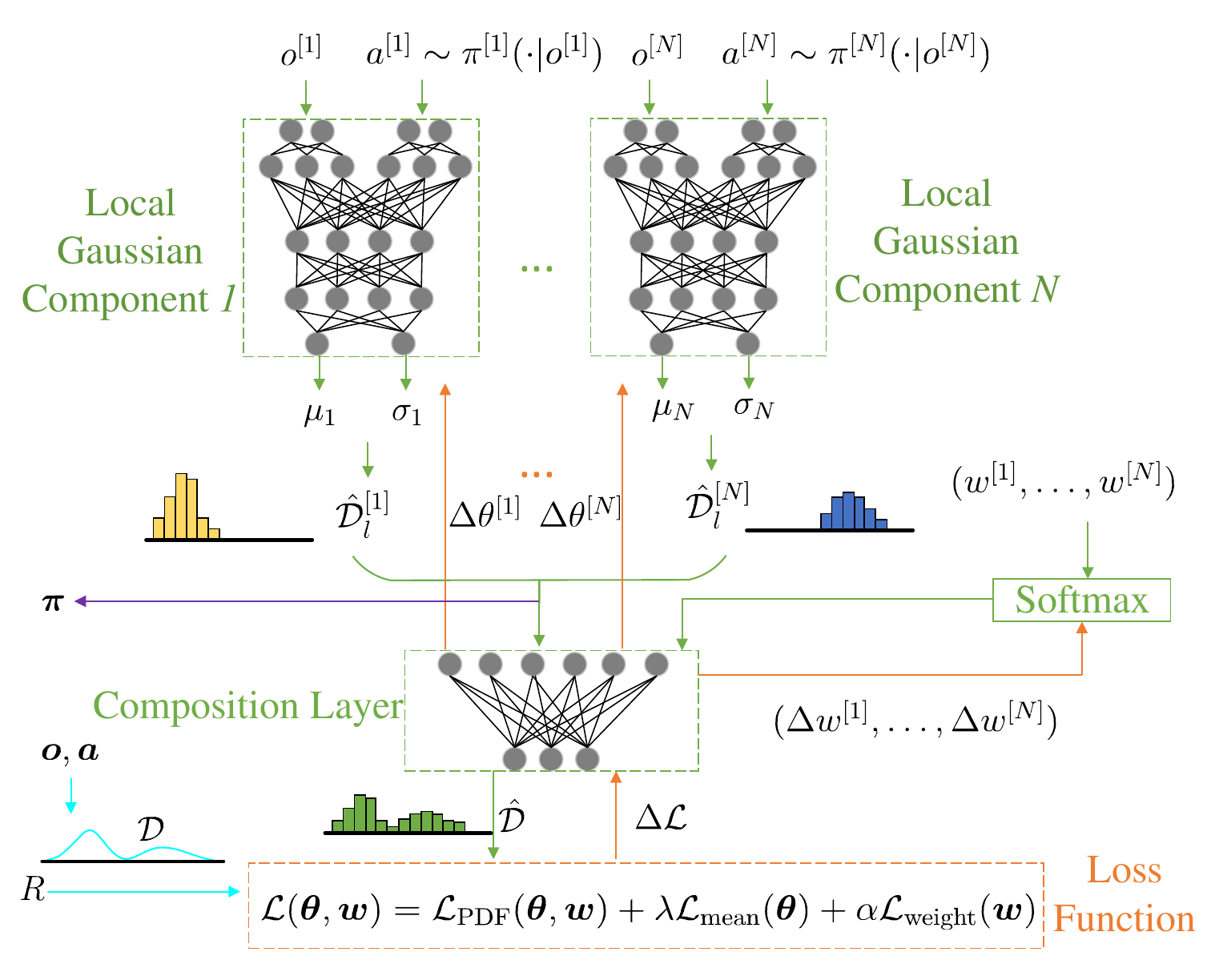}}
\caption{Overview of the NDD algorithm.}
\vspace{-0.3cm}
\label{process_fig}
\end{figure}

In a nutshell, for the NDD algorithm depicted in Fig. \ref{process_fig}, we first predict local Gaussian components that will be utilized for local training by a multi-layer perceptron (MLP), and aggregate these decomposed ones through another MLP-based composition layer, so as to approximate the GMM along original and augmented rewards. Next we update $\bm \theta$ and $\bm w$ according to \eqref{eq:total_loss} and learn $\pi^{[i]}$ with $\hat{\mathcal{D}}_{l}^{[i]}(\theta^{[i]})$ through distributional RL independently. 
% \blus{Every agent $i$ in MAS updates its critic networks with decomposed $R^{[i]}$ and its policy $\pi^{[i]}$'s networks by generalized advantage estimation (GAE)\cite{yu2022surprising, schulman2017proximal}.} 
Details are summarized in Algorithm \ref{algorithm}. 

\section{Experimental Settings and Simulation Results}\label{sectionRes}

\subsection{Experimental Settings}
\label{sec:exp_setting}
We evaluate NDD in Multi-agent Particle world Environments (MPE) and StarCraft Multi-Agent Challenge (SMAC) \cite{mordatch2018emergence,lowe2017multi,samvelyan19smac}, which is implemented on top of the current SOTA algorithm MultiAgent Proximal Policy Optimization (MAPPO) \cite{de2020independent, schulman2017proximal,yu2022surprising}. MPE is a classical MARL environment with each particle emulating the potential behaviors of one agent. Typically, MPE consists of several task-oriented collaboration scenarios (e.g., Adversary, Crypto, Speaker-Listener, Spread and Reference). 
Compared to MPE, SMAC is a more complicated environment based on Blizzard's StarCraft II RTS game.
% and encompasses many evaluation cases as well. 
Notably, as shown in Table \ref{tab:noise_config_along_scene}, we add $5$ types of noise (i.e., Noise $0$ to Noise $5$) to the globally shared rewards, which are different in MPE and SMAC scenarios so as to be in line with the distinct reward values (i.e., one to two orders of magnitude smaller in SMAC than MPE).  
Additionally, we choose some typical actor cirtic algorithms MADDPG \cite{lowe2017multi}, MATD3\cite{td3} and value decomposition algorithms VDN\cite{sunehag2017value}, QMIX\cite{rashid2020monotonic} as well as MAPPO for comparison. Furthermore, we choose MAPPO under noiseless conditions as the ``Baseline'' to measure the performance loss of NDD and the others in noisy environments. 
\blut{Considering that the number of iterations is generally based on empirical values and changes over environments, }
every method is trained with $3,000$ iterations for each task in MPE, and $2\times 10^4$ iterations for each task in SMAC.
Afterwards, we compare the performance under different types of zero-mean noises to rule out the possible bias impact of noise. \blus{More experimental  configurations are given in Appendix \ref{hyper_parameters_setting}.}

\begin{table*}[htbp]
\caption{Performance of NDD and Others with Five Types of Noises in MPE Tasks. The ``Baseline'' is the result of MAPPO without Noise.}
\label{main_res_mpe}
\begin{center}
\footnotesize
\setlength{\tabcolsep}{1mm}{
\begin{tabular}{|c|c|c|c|c|c|c|c|c|}
\toprule
\textbf{Scenario} & \textbf{Type}&\textbf{VDN}&\textbf{QMIX}& \textbf{MATD3}& \textbf{MADDPG}&\textbf{MAPPO}&\textbf{NDD}&\textbf{Baseline}\\
\midrule
\multirow{5}{*}{Adversary}&Noise $0$&${-}25.63$$\pm$$0.78$&${-}37.39$$\pm$$3.40$& $-9.67\pm3.50$& $-14.81\pm1.80$&${-}8.23$$\pm$$0.66$&$\bm{{-}1.47}$$\pm$$\bm{0.49}$&${-}2.39$$\pm$$0.22$\\
&Noise $1$&${-}24.24$$\pm$$1.13$&${-}46.17$$\pm$$5.61$& $-15.88\pm2.19$& $-13.32\pm2.43$&${-}6.16$$\pm$$0.78$&$\bm{{-}1.51}$$\pm$$\bm{0.19}$&${-}1.91$$\pm$$0.35$\\
&Noise $2$&${-}23.72$$\pm$$1.57$&${-}40.39$$\pm$$17.30$& $-16.21\pm3.12$& $-19.33\pm0.42$&${-}7.76$$\pm$$0.78$&$\bm{{-}0.89}$$\pm$$\bm{0.24}$&${-}1.82$$\pm$$0.10$\\
& Noise $3$& ${-}18.12\pm0.99$& ${-}43.18\pm6.51$& ${-}6.97\pm1.27$& ${-}18.49\pm3.67$& ${-}4.90\pm0.22$& $\bm{{-}1.42\pm0.75}$& ${-}1.74\pm0.69$\\
& Noise $4$& ${-}25.72\pm0.36$& ${-}25.79\pm3.89$& ${-}11.04\pm10.21$& ${-}53.30\pm32.57$& ${-}9.72\pm0.62$& $\bm{{-}4.74\pm1.19}$& ${-}0.93\pm0.19$\\
\hline
\multirow{5}{*}{Crypto}&Noise $0$&${-}66.97$$\pm$$3.03$&${-}69.56$$\pm$$0.94$& $-13.40\pm7.11$& $\bm{-1.20\pm28.79}$&${-}28.90$$\pm$$1.00$&${-}11.77$$\pm$$2.00$&${-}3.68$$\pm$$0.97$\\
&Noise $1$&${-}69.11$$\pm$$3.29$&${-}55.12$$\pm$$13.06$& $\bm{1.20\pm16.29}$& $-4.21\pm4.08$&${-}22.51$$\pm$$0.47$&${-}5.28$$\pm$$0.68$&${-}8.67$$\pm$$0.63$\\
&Noise $2$&${-}68.29$$\pm$$0.38$&${-}69.35$$\pm$$2.17$& $-20.45\pm21.56$& $-17.84\pm15.95$&${-}24.84$$\pm$$1.01$&$\bm{{-}10.69}$$\pm$$\bm{0.08}$&${-}2.43$$\pm$$0.13$\\
& Noise $3$& ${-}67.46\pm2.04$& ${-}49.68\pm12.31$& $\bm{{-}0.05\pm2.63}$& ${-}12.19\pm3.29$& ${-}19.60\pm1.42$& ${-}19.47\pm1.25$& ${-}1.31\pm0.08$\\
& Noise $4$& ${-}71.94\pm3.17$& ${-}65.22\pm3.35$& ${-}0.60\pm21.54$& $\bm{0.60\pm4.39}$& ${-}28.56\pm3.32$& ${-}27.36\pm1.10$& ${-}5.43\pm2.35$\\
\hline
\multirow{3}{*}{Speaker}&Noise $0$&${-}34.98$$\pm$$1.32$&${-}33.66$$\pm$$0.86$& $-53.08\pm13.12$& $-46.36\pm1.45$&$\bm{{-}31.81}$$\pm$$\bm{3.67}$&${-}33.40$$\pm$$1.39$&${-}31.80$$\pm$$2.79$\\
&Noise $1$&${-}34.52$$\pm$$0.99$&$\bm{{-}28.75}$$\pm$$\bm{2.42}$& $-48.16\pm15.37$& $-45.93\pm1.05$&${-}34.03$$\pm$$0.33$&${-}29.76$$\pm$$0.93$&${-}32.66$$\pm$$2.46$\\
\multirow{2}{*}{Listener}&Noise $2$&${-}34.53$$\pm$$1.45$&${-}34.06$$\pm$$2.19$& $-51.03\pm11.95$& $-41.52\pm0.87$&${-}34.88$$\pm$$2.74$&$\bm{{-}32.42}$$\pm$$\bm{0.96}$&${-}28.95$$\pm$$0.47$\\
& Noise $3$& ${-}35.51\pm0.42$& $\bm{{-}28.61\pm0.77}$& ${-}62.00\pm4.94$& ${-}44.26\pm9.32$& ${-}32.63\pm0.78$& ${-}33.22\pm0.45$& ${-}31.66\pm0.41$\\
& Noise $4$& ${-}34.90\pm1.99$& ${-}38.01\pm0.64$& ${-}67.28\pm1.18$& ${-}47.03\pm1.56$& ${-}37.33\pm2.17$& $\bm{{-}33.70\pm0.28}$& ${-}30.26\pm2.05$\\
\hline
\multirow{5}{*}{Spread}&Noise $0$&${-}104.84$$\pm$$0.36$&${-}171.02$$\pm$$7.04$& $-114.46\pm3.06$& $-106.85\pm0.64$&$\bm{{-}93.58}$$\pm$$\bm{0.68}$&${-}94.94$$\pm$$0.52$&${-}91.64$$\pm$$0.69$\\
&Noise $1$&${-}98.90$$\pm$$0.87$&${-}156.00$$\pm$$19.36$& $-104.02\pm3.68$& $-102.49\pm0.37$&${-}97.86$$\pm$$5.23$&$\bm{{-}93.66}$$\pm$$\bm{2.27}$&${-}94.34$$\pm$$0.83$\\
&Noise $2$&${-}97.59$$\pm$$1.32$&${-}171.36$$\pm$$16.49$& $-103.27\pm1.68$& $-104.44\pm2.58$&${-}97.91$$\pm$$1.05$&$\bm{{-}91.58}$$\pm$$\bm{0.75}$&${-}92.62$$\pm$$0.87$\\
& Noise $3$& ${-}95.05\pm0.49$& ${-}139.44\pm11.54$& ${-}105.26\pm0.46$& ${-}101.05\pm3.19$& $\bm{{-}91.97\pm1.07}$& ${-}94.54\pm0.19$& ${-}93.19\pm0.73$\\
& Noise $4$& ${-}98.37\pm1.83$& ${-}190.89\pm7.23$& ${-}99.37\pm1.05$& ${-}107.87\pm1.31$& ${-}96.00\pm1.66$& $\bm{{-}93.29\pm1.35}$& ${-}91.33\pm0.28$\\
\hline
\multirow{5}{*}{Reference}&Noise $0$&${-}43.28$$\pm$$0.67$&${-}76.25$$\pm$$4.97$& $-48.16\pm1.58$& $-52.35\pm0.61$&${-}40.65$$\pm$$0.52$&$\bm{{-}36.33}$$\pm$$\bm{0.35}$&${-}36.14$$\pm$$0.11$\\%
&Noise $1$&${-}40.49$$\pm$$0.62$&${-}58.67$$\pm$$9.02$& $-68.77\pm14.50$& $-51.22\pm1.32$&${-}39.07$$\pm$$0.35$&$\bm{{-}37.44}$$\pm$$\bm{0.92}$&${-}36.09$$\pm$$0.73$\\
&Noise $2$&${-}42.41$$\pm$$0.73$&${-}62.82$$\pm$$10.49$& $-49.53\pm1.35$& $-46.13\pm2.48$&${-}37.42$$\pm$$0.66$&$\bm{{-}37.29}$$\pm$$\bm{1.54}$&${-}36.25$$\pm$$0.20$\\
& Noise $3$& ${-}40.05\pm0.13$& ${-}66.59\pm8.74$& ${-}47.79\pm0.33$& ${-}48.05\pm2.68$& ${-}38.88\pm0.44$& $\bm{{-}35.76\pm0.55}$& ${-}36.49\pm1.02$\\
& Noise $4$& ${-}42.56\pm0.19$& ${-}72.21\pm4.53$& ${-}49.18\pm4.54$& ${-}50.92\pm1.72$& ${-}36.57\pm0.73$& $\bm{{-}36.26\pm1.02}$& ${-}35.62\pm0.16$\\
\bottomrule
\end{tabular}}
\end{center}
\end{table*}

\begin{table*}[htbp]
\caption{Performance of NDD and Others with Five Types of Noises in SMAC Tasks. The ``Baseline'' is the result of MAPPO without Noise. ``3m" means 3 Marines in MAS, and ``z, s, sc" in tasks' names are ``Zealot, Stalker, Spine Crawler", respectively.}
\label{main_res_smac}
\begin{center}
\footnotesize
\begin{tabular}{|c|c|c|c|c|c|c|c|c|}
\toprule
\textbf{Scenario} & \textbf{Type}&\textbf{VDN}&\textbf{QMIX}& \textbf{MATD3}& \textbf{MADDPG}&\textbf{MAPPO}&\textbf{NDD}&\textbf{Baseline}\\
\midrule
\multirow{5}{*}{3m}&Noise $0$&$4.69$$\pm$$2.22$&$5.74$$\pm$$1.17$& $0.39\pm1.16$& $-0.14\pm0.32$&$8.27$$\pm$$3.64$&$\bm{11.64}$$\pm$$\bm{0.83}$&\multirow{5}{*}{$12.11$$\pm$$0.30$}\\
&Noise $1$&$5.92$$\pm$$1.12$&$6.52$$\pm$$1.63$& $-0.16\pm0.35$& $0.10\pm0.45$&$8.07$$\pm$$2.25$&$\bm{11.49}$$\pm$$\bm{1.16}$&\\
&Noise $2$&$4.98$$\pm$$2.31$&$2.37$$\pm$$0.74$& $0.06\pm0.41$& $0.01\pm0.53$&$\bm{11.84}$$\pm$$\bm{1.35}$&$8.47$$\pm$$1.73$&\\
& Noise $3$& $7.06\pm1.84$& $7.75\pm1.08$& $0.13\pm0.41$& ${-}0.02\pm0.16$& $\bm{11.63\pm1.10}$& $11.12\pm0.73$&\\
& Noise $4$& $7.70\pm1.67$& $6.41\pm1.61$& ${-}3.33\pm0.36$& ${-}3.62\pm0.25$& $10.91\pm0.77$& $\bm{12.79\pm0.68}$&\\
\hline
\multirow{5}{*}{8m}&Noise $0$&$\bm{4.55}$$\pm$$\bm{0.75}$&$3.39$$\pm$$1.60$& $1.14\pm1.47$& $0.24\pm1.60$&$0.42$$\pm$$2.51$&$0.48$$\pm$$0.49$&\multirow{5}{*}{$4.98$$\pm$$0.40$}\\
&Noise $1$&$2.97$$\pm$$0.60$&$6.42$$\pm$$0.93$& $0.45\pm0.50$& $0.80\pm0.49$&$0.86$$\pm$$1.72$&$\bm{6.84}$$\pm$$\bm{4.13}$&\\
&Noise $2$&$2.57$$\pm$$0.61$&$3.92$$\pm$$0.64$& $0.72\pm0.61$& $0.59\pm0.70$&$3.13$$\pm$$1.36$&$\bm{10.90}$$\pm$$\bm{0.81}$&\\
& Noise $3$& $\bm{3.26\pm0.75}$& $2.60\pm0.73$& $0.56\pm0.37$& $0.26\pm0.40$& $1.19\pm1.25$& $2.95\pm1.38$&\\
& Noise $4$& $0.82\pm0.96$& $2.80\pm3.28$& ${-}6.73\pm0.24$& ${-}6.71\pm0.32$& $3.91\pm1.76$& $\bm{8.38\pm2.25}$&\\
\hline
\multirow{5}{*}{2m vs 1z}&Noise $0$&$2.60$$\pm$$0.77$&$2.87$$\pm$$0.68$& $0.92\pm1.45$& $0.35\pm1.04$&$3.51$$\pm$$0.17$&$\bm{3.65}$$\pm$$\bm{0.09}$&\multirow{5}{*}{$3.32$$\pm$$0.01$}\\
&Noise $1$&$2.62$$\pm$$1.20$&$2.70$$\pm$$0.44$& $0.21\pm0.77$& $-0.06\pm0.46$&$3.40$$\pm$$0.23$&$\bm{3.46}$$\pm$$\bm{0.06}$&\\
&Noise $2$&$2.67$$\pm$$0.74$&$2.41$$\pm$$0.79$& $0.34\pm0.53$& $0.20\pm0.99$&$3.31$$\pm$$0.47$&$\bm{3.59}$$\pm$$\bm{0.05}$&\\
& Noise $3$& $1.97\pm1.01$& $2.28\pm0.45$& $0.41\pm0.51$& $0.14\pm0.29$& $3.18\pm0.69$& $\bm{3.68\pm0.08}$&\\
& Noise $4$& ${-}7.74\pm1.11$& ${-}6.56\pm0.58$& ${-}8.96\pm0.64$& ${-}9.01\pm0.53$& $2.28\pm0.26$& $\bm{3.55\pm0.04}$&\\
\hline
\multirow{5}{*}{2s3z}&Noise $0$&$2.72$$\pm$$0.98$&$3.28$$\pm$$2.43$& $2.12\pm1.12$& $1.68\pm1.24$&$2.97$$\pm$$1.00$&$\bm{7.72}$$\pm$$\bm{0.40}$&\multirow{5}{*}{$6.83$$\pm$$0.33$}\\
&Noise $1$&$3.11$$\pm$$0.63$&$2.29$$\pm$$0.61$& $0.71\pm0.47$& $0.70\pm0.38$&$1.52$$\pm$$0.33$&$\bm{7.21}$$\pm$$\bm{0.46}$&\\
&Noise $2$&$3.36$$\pm$$0.69$&$2.02$$\pm$$0.47$& $1.16\pm0.67$& $0.84\pm0.71$&$6.51$$\pm$$1.75$&$\bm{7.49}$$\pm$$\bm{0.29}$&\\
& Noise $3$& $3.33\pm0.86$& $2.43\pm0.60$& $1.13\pm0.25$& $0.42\pm0.39$& $3.40\pm3.01$& $\bm{5.94\pm0.78}$&\\
& Noise $4$& ${-}1.60\pm0.48$& ${-}1.33\pm0.76$& ${-}5.48\pm0.80$& ${-}6.06\pm1.08$& $2.75\pm1.07$& $\bm{6.91\pm0.39}$&\\
\hline
\multirow{5}{*}{2s vs 1sc}&Noise $0$&$0.14$$\pm$$1.22$&${-}0.14$$\pm$$1.16$&$ 0.13\pm1.50$& $-0.17\pm0.37$&$6.42$$\pm$$0.59$&$\bm{6.46}$$\pm$$\bm{0.02}$&\multirow{5}{*}{$6.44$$\pm$$0.01$}\\
&Noise $1$&$0.37$$\pm$$1.22$&${-}0.67$$\pm$$0.98$& $-0.13\pm1.04$& $-0.07\pm0.99$&$6.25$$\pm$$0.28$&$\bm{6.44}$$\pm$$\bm{0.02}$&\\
&Noise $2$&$0.07$$\pm$$1.51$&$0.49$$\pm$$0.69$& $0.02\pm1.10$& $0.29\pm0.79$&$6.29$$\pm$$0.43$&$\bm{6.44}$$\pm$$\bm{0.03}$&\\
& Noise $3$& ${-}0.05\pm0.91$& $0.20\pm0.71$& $0.14\pm0.49$& ${-}0.13\pm0.28$& $6.39\pm0.27$& $\bm{6.43\pm0.03}$&\\
& Noise $4$& ${-}17.70\pm0.59$& ${-}10.33\pm7.55$& ${-}17.62\pm0.39$& ${-}10.28\pm7.58$& $3.78\pm0.19$& $\bm{6.45\pm0.02}$&\\
\hline
\multirow{5}{*}{3s vs 3z}&Noise $0$&$3.20$$\pm$$1.46$&$0.91$$\pm$$1.64$& $2.47\pm1.01$& $1.90\pm1.30$&$2.76$$\pm$$1.84$&$\bm{4.79}$$\pm$$\bm{0.16}$&\multirow{5}{*}{$0.00$$\pm$$0.00$}\\
&Noise $1$&$\bm{3.20}$$\pm$$\bm{1.57}$&$1.49$$\pm$$1.54$& $1.70\pm0.80$& $1.84\pm0.63$&$0.48$$\pm$$0.68$&$2.08$$\pm$$2.06$&\\
&Noise $2$&$1.38$$\pm$$0.95$&$0.69$$\pm$$0.95$& $2.07\pm0.83$& $1.97\pm0.62$&$3.97$$\pm$$1.0$8&$\bm{4.83}$$\pm$$\bm{0.08}$&\\
& Noise $3$& $2.71\pm1.14$& $0.45\pm0.71$& $1.81\pm0.46$& $1.27\pm0.33$& $1.96\pm2.43$& $\bm{4.28\pm0.67}$&\\
& Noise $4$& ${-}4.34\pm1.04$& ${-}6.92\pm2.31$& ${-}6.89\pm0.55$& ${-}7.60\pm0.38$& $2.35\pm0.22$& $\bm{4.84\pm0.18}$&\\
\hline
\multirow{5}{*}{3s vs 4z}&Noise $0$&$2.52$$\pm$$1.13$&$\bm{3.01}$$\pm$$\bm{0.80}$& $-0.26\pm1.46$& $2.21\pm1.55$&$0.38$$\pm$$1.27$&$2.96$$\pm$$0.03$&\multirow{5}{*}{$2.96$$\pm$$0.04$}\\
&Noise $1$&$2.78$$\pm$$0.29$&$\bm{2.89}$$\pm$$\bm{0.44}$& $0.04\pm1.19$& $1.70\pm1.31$&$0.43$$\pm$$0.47$&$2.80$$\pm$$0.04$&\\
&Noise $2$&$2.62$$\pm$$0.50$&$2.13$$\pm$$0.86$& $0.59\pm1.07$& $0.34\pm1.01$&$0.11$$\pm$$0.68$&$\bm{2.89}$$\pm$$\bm{0.09}$&\\
& Noise $3$& $2.46\pm0.35$& $\bm{2.87\pm0.20}$& $0.24\pm0.29$& $0.57\pm0.61$& $2.45\pm0.51$& $2.52\pm0.35$&\\
& Noise $4$& ${-}4.36\pm3.41$& ${-}0.81\pm0.96$& ${-}12.06\pm0.55$& ${-}11.00\pm1.07$& $0.87\pm0.17$& $\bm{2.80\pm0.03}$&\\
\hline
\multirow{5}{*}{3s5z}&Noise $0$&$2.81$$\pm$$1.02$&$4.13$$\pm$$1.28$& $1.42\pm1.32$& $1.37\pm1.78$&$3.24$$\pm$$1.56$&$\bm{5.71}$$\pm$$\bm{0.13}$&\multirow{5}{*}{$5.67$$\pm$$0.05$}\\
&Noise $1$&$2.77$$\pm$$0.27$&$3.38$$\pm$$0.46$& $1.91\pm0.75$& $2.20\pm0.73$&$3.69$$\pm$$0.75$&$\bm{5.27}$$\pm$$\bm{0.37}$&\\
&Noise $2$&$2.89$$\pm$$0.42$&$3.91$$\pm$$0.85$& $2.57\pm0.61$& $1.46\pm0.64$&$3.24$$\pm$$1.04$&$\bm{5.75}$$\pm$$\bm{0.24}$&\\
& Noise $3$& $3.09\pm0.50$& $\bm{3.80\pm0.30}$& $1.79\pm0.68$& $2.25\pm0.86$& $1.85\pm1.60$& $3.23\pm0.38$&\\
& Noise $4$& ${-}1.84\pm0.65$& ${-}1.49\pm0.64$& ${-}6.05\pm0.74$& ${-}6.81\pm1.02$& $0.74\pm1.45$& $\bm{5.82\pm0.17}$&\\
\hline
\multirow{5}{*}{5m vs 6m}&Noise $0$&$2.12$$\pm$$1.02$&$2.08$$\pm$$0.91$& $1.64\pm1.02$& $0.31\pm0.90$&$2.40$$\pm$$2.10$&$\bm{4.83}$$\pm$$\bm{0.23}$&\multirow{5}{*}{$4.60$$\pm$$0.68$}\\
&Noise $1$&$2.86$$\pm$$0.40$&$\bm{3.71}$$\pm$$\bm{0.25}$& $0.79\pm0.66$& $0.18\pm0.34$&$2.59$$\pm$$2.02$&$2.20$$\pm$$2.19$&\\
&Noise $2$&$3.04$$\pm$$0.50$&$\bm{3.12}$$\pm$$\bm{0.89}$& $0.97\pm0.64$& $0.41\pm0.27$&$2.49$$\pm$$0.50$&$2.89$$\pm$$0.64$&\\
& Noise $3$& $2.36\pm0.45$& $\bm{4.10\pm0.20}$& $0.82\pm0.24$& $0.05\pm0.28$& $0.89\pm0.89$& $2.87\pm2.13$&\\
& Noise $4$& $2.68\pm0.52$& $2.74\pm0.36$& ${-}3.89\pm0.22$& ${-}4.25\pm0.22$& $3.34\pm0.16$& $\bm{4.83\pm0.11}$&\\
\hline
\multirow{5}{*}{8m vs 9m}&Noise $0$&$0.74$$\pm$$1.05$&$\bm{2.91}$$\pm$$\bm{0.78}$& $0.78\pm1.31$& $0.06\pm1.55$&$1.89$$\pm$$1.73$&$1.54$$\pm$$1.51$&\multirow{5}{*}{$4.49$$\pm$$0.16$}\\
&Noise $1$&$2.43$$\pm$$0.24$&$1.64$$\pm$$0.48$& $0.58\pm0.77$& $0.27\pm0.88$&$2.72$$\pm$$0.68$&$\bm{3.57}$$\pm$$\bm{0.18}$&\\
&Noise $2$&$2.12$$\pm$$0.42$&$2.00$$\pm$$0.64$& $0.48\pm0.86$& $0.56\pm0.43$&${-}0.15$$\pm$$0.70$&$\bm{2.78}$$\pm$$\bm{1.18}$&\\
& Noise $3$& $1.74\pm0.35$& $3.29\pm2.03$& $0.48\pm0.46$& $0.17\pm0.32$& $1.35\pm1.38$& $\bm{3.63\pm0.44}$&\\
& Noise $4$& $0.44\pm2.71$& $1.40\pm1.87$& ${-}6.69\pm0.30$& ${-}6.40\pm0.43$& $4.75\pm0.39$& $\bm{5.55\pm0.50}$&\\
\bottomrule
\end{tabular}
\end{center}
\end{table*}
\subsection{Performance Comparison} 
\begin{figure}
    \centering
    \includegraphics[width=.45\textwidth]{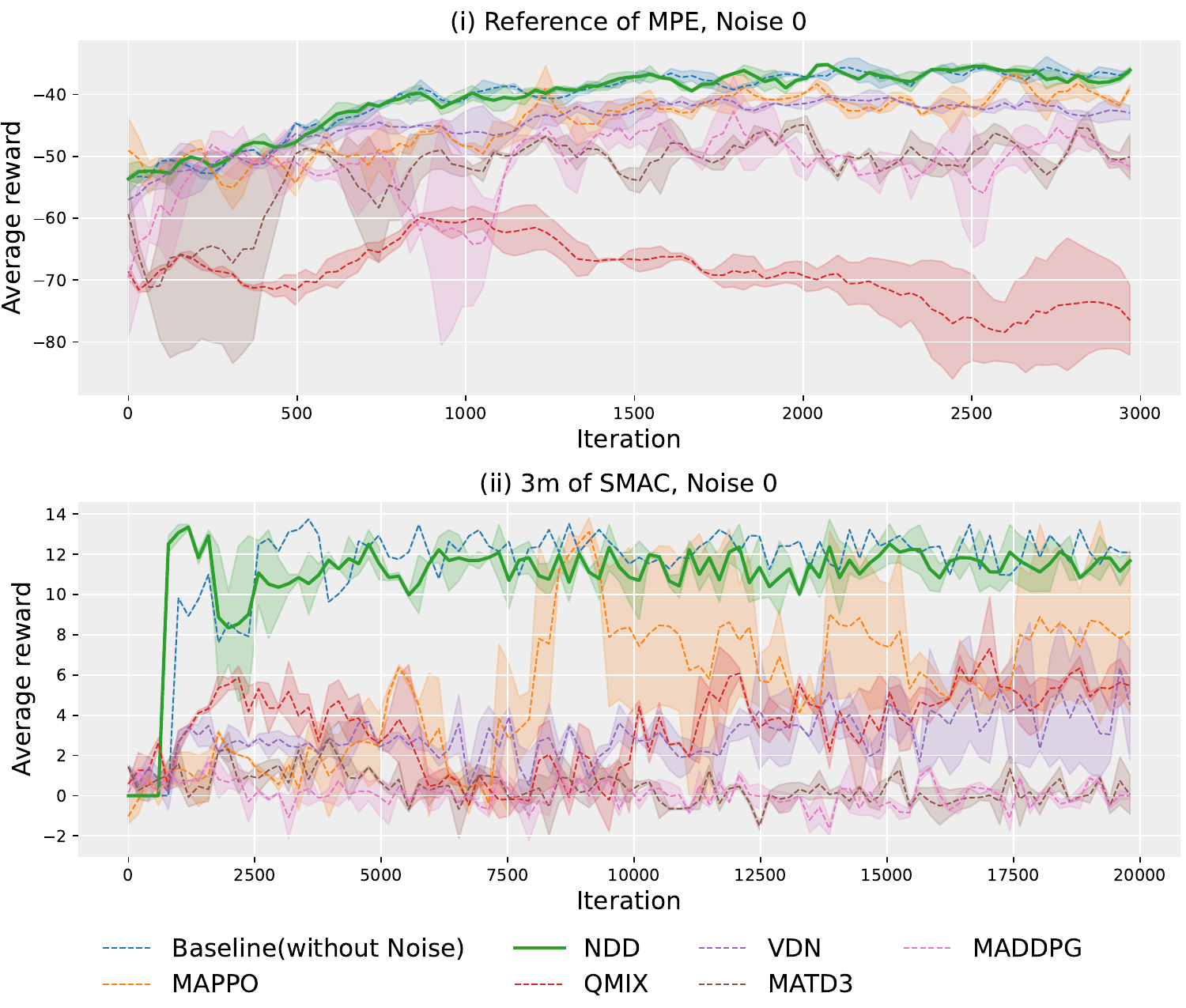}
    \caption{ Performance comparison between algorithms in (\romannumeral1) the MPE Reference task with Noise $0$; (\romannumeral2) the SMAC 3m task with Noise $0$. Notably, the ``Baseline'' indicates the result of MAPPO under noise-free settings; \blus{``3m'' means 3 Marines in MAS.}}
    \label{main_res_1}
\end{figure}
\begin{figure}
    \centering
    \includegraphics[width=.45\textwidth]{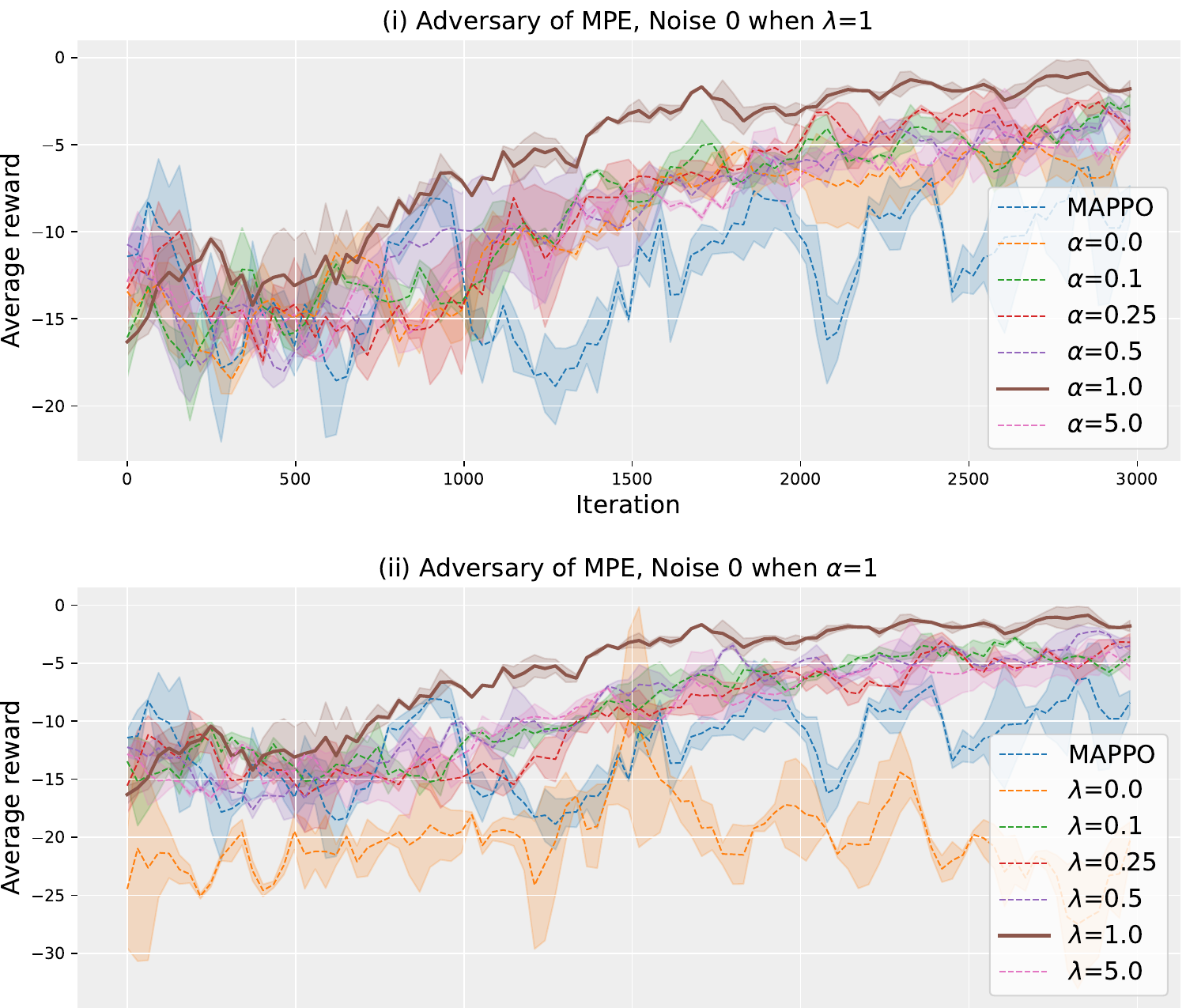}
    \caption{Performance sensitivity of NDD in Adversary with Noise $0$ under the settings of (\romannumeral1) different $\alpha$ with $\lambda=1$; (\romannumeral2) different $\lambda$ with $\alpha=1$.}
    \label{lambdaalpha_fig}
\end{figure}

Fig. \ref{main_res_1} compares NDD with the aforementioned methods and visualizes the corresponding learning process of agents for one MPE task and one SMAC task. Meanwhile,
Table \ref{main_res_mpe} and \ref{main_res_smac} summarize the performance comparisons in all noisy tasks considering five distinctive types of noise as in Table \ref{tab:noise_config_along_scene}. 
It's worth noting that there are slight differences in the five simulations' results of the ``Baseline'' over five types of noise for each task in MPE, as the states of agents are randomly initialized at the beginning of each episode, which is in stark contrast to SMAC.
From the results, NDD is significantly superior to the others and exhibits a satisfactory anti-noise effect.
For example, in Reference, the interference of noise produces a significant negative impact on experimental results, especially for the value decomposition algorithms. Accordingly, it can be observed from Fig. \ref{main_res_1} that VDN, QMIX, MATD3, MADDPG, and MAPPO fluctuate heavily and converge to lower values. Regardless of the noise distribution, NDD stabilizes the performance and generates results similar to noise-free tasks in most scenarios. By the way, the effect of QMIX is much worse than other methods, due to that the mixing network used in QMIX requires more stable rewards until convergence. Furthermore, in noisy scenarios, a complex mixing network yields less robust results than the one with simple summation. Hence, we primarily choose GMM due to its simplicity and robustness. 
% In Table \ref{main_res_mpe} Crypto and Speaker Listener, some other algorithms achieve higher performance than NDD, but fluctuated greatly. 
In addition, all algorithms exhibit performance decreases in complex noises. Due to the simultaneous involvement of multiple noise sources and non-Gaussian components, the results for Noise $3$ and Noise $4$ are worse than others, but the NDD yields more comparable performance to that obtained in noise-free cases.

Fig. \ref{lambdaalpha_fig} provides the sensitivity study about $\lambda$ and $\alpha$. Additionally, we choose the high-performance MAPPO for comparison. It can be observed that $\lambda=1$ is the optimal value and either too large or too small may cause performance degradation. Similar observations can be applied to $\alpha$. Additionally, NDD is relatively more sensitive to $\lambda$ compared with $\alpha$, and the performance when $\lambda=0$ (i.e., no $\mathcal{L}_{\text{mean}}(\bm \theta)$ in the loss function in \eqref{eq:total_loss}) is even worse than MAPPO. In other words, $\mathcal{L}_{\text{mean}}(\bm \theta)$ is of vital importance to the NDD.

\begin{table*}[htbp]\label{ndd_evalue}
\footnotesize
\setlength{\tabcolsep}{0.65mm}{}
\caption{Modeling Accuracy of Noise Distribution Decomposition under Different GMM Assumptions. $\mathcal{N}$ Indicates a Gaussian Distribution, while $\beta$ and $\mathcal{X}^2$ Denote a Beta and a Chi-Square Distribution, Respectively.}
\vspace{-0.3cm}
\label{tab1}
\begin{center}
\begin{tabular}{|c|c|c|c|c|c|}
\toprule
\multirow{2}{*}{\textbf{Noise Distribution}} \textbf{}&\multicolumn{3}{|c|}{\textbf{Decomposed Gaussian Components}}&\multirow{2}{*}{\textbf{Weights}}&\textbf{Wasserstein}\\
\cline{2-4} 
& \textbf{\textit{Part 1}}& \textbf{\textit{Part 2}}&\textbf{\textit{Part 3}}&& \textbf{Distance} $d_p$ \\
\midrule
\multirow{2}{*}{$\mathcal{N}(0,5)$} & $\mathcal{N}(5.4\pm 0.48,$ & $\mathcal{N}(-0.47\pm 0.20,$ & $\mathcal{N}(-4.4\pm 0.14,$ & $[0.35\pm0.016,0.30\pm$ &\multirow{2}{*}{$4.1\pm 0.038\times 10^{-3}$}\\
& $11\pm 1.2$) & $6.7\pm 0.82$) & $19\pm 0.97$) & $0.029, 0.34\pm 0.0051]$ &\\
\hline
\multirow{2}{*}{$\mathcal{N}$(0,3)}& $\mathcal{N}$($2.8\pm 0.14,$ & $\mathcal{N}$($-0.48\pm 0.11,$ & $\mathcal{N}$($-2.3\pm 0.28,$ & [$0.30\pm0.029,0.35\pm$ &\\
& $3.3\pm 0.25$) & $2.7\pm 0.23$) & $8.2\pm 1.4$) & $0.025, 0.35\pm 0.016$] & $3.8\pm 0.12\times 10^{-3}$\\
\hline
$0.5\mathcal{N}(1,5)+$& $\mathcal{N}(4.4\pm 2.4,$ & $\mathcal{N}(-0.10\pm 1.4,$ & $\mathcal{N}(-2.5\pm 1.1,$ & $[0.30\pm0.059,0.34\pm$ &\multirow{2}{*}{$3.9\pm 0.072\times 10^{-3}$}\\
$0.5\mathcal{N}(-1,5)$& $22\pm 10)$ & $21\pm 13)$ & $13\pm 0.071)$ & $0.017, 0.35\pm 0.042]$ &\\
\hline
$0.4\mathcal{N}$(5,1)+& $\mathcal{N}$($4.9\pm 0.0030,$ & $\mathcal{N}$($-3.6\pm 0.042,$ & $\mathcal{N}$($-5.4\pm 0.036,$ & [$0.34\pm0.017,0.30\pm$ &\\
$0.6\mathcal{N}$(-5,3)& $0.87\pm 0.046$) & $32\pm 2.0$) & $5.2\pm 0.67$) & $0.059, 0.35\pm 0.042$] & $4.1\pm 0.74\times 10^{-3}$\\
\hline
$0.25\beta$(1,2)+& $\mathcal{N}$($0.29\pm 0.10,$ & $\mathcal{N}$($-4.2\pm 0.057,$ & $\mathcal{N}$($-7.0\pm 0.27,$ & [$0.33\pm0.00093,0.33\pm$ &\\
$0.75\mathcal{N}$(-5,3)& $0.26\pm 0.14$) & $3.4\pm 0.60$) & $6.5\pm 0.24$) & $0.00047, 0.33\pm 0.0014$] & $4.5\pm 0.32\times 10^{-3}$\\
\hline
$0.3\mathcal{N}(-1,5)+$& $\mathcal{N}(5.9\pm 0.32,$ & $\mathcal{N}(-0.19\pm 0.31,$ & $\mathcal{N}(-4.0\pm 0.20,$ & $[0.32\pm0.026,0.33\pm$ &\multirow{2}{*}{$3.5\pm 0.033\times 10^{-3}$}\\
$0.2\mathcal{N}(-2,5)$+$0.5\mathcal{N}(1.4,5)$& $13\pm 1.1)$ & $10\pm 0.13)$ & $15\pm 0.020)$ & $0.0038, 0.35\pm 0.030]$ &\\
\hline
$0.35\mathcal{N}$(-6,1)+& $\mathcal{N}$($9.3\pm 0.093,$ & $\mathcal{N}$($0.30\pm 0.061,$&$\mathcal{N}$($-5.9\pm 0.0018,$ & [$0.35\pm0.037,0.32\pm$&\multirow{2}{*}{$8.8\pm0.64\times 10^{-3}$}\\
$0.3\beta(1,2)+0.35\mathcal{\chi^{\text{2}}}(9)$& $0.60\pm 0.082$) & $0.19\pm 0.062$) & $0.95\pm 0.039$) & $0.026, 0.33\pm 0.011$] &\\
\bottomrule
\end{tabular}
\vspace{-0.5cm}
\end{center}
\end{table*}

% \begin{figure}[t] 
% \centering
% \begin{subfigure}{}
%     \includegraphics[width=0.145\textwidth]{case_5_3agents_NDD_curve.pdf}%height=0.15\textheight
%     \label{fig_NDDevaluate_0}
% \end{subfigure}
% \begin{subfigure}{}
%     \includegraphics[width=0.145\textwidth]{case_6_3agents_NDD_curve.pdf}
%     \label{fig_NDDevaluate_1}
% \end{subfigure}
% \begin{subfigure}{}
%     \includegraphics[width=0.145\textwidth]{Loss_curve.pdf}
%     \label{fig_NDDevaluate_2}
% \end{subfigure}
% \caption{Results of distribution decomposition over two noise cases (i.e., $0.25\beta(1,2) + 0.75\mathcal{N}(-5,3)$, $0.35\mathcal{N}(-6,1)+0.3\beta(1,2)+0.35\chi(9)$). The left two sub-figures show the comparison between practical PDF histogram and weighted PDFs of decomposed distributions ($3$ colored curves); while the right sub-figure shows the curve of loss over training iterations. $\mathcal{N}, \beta, \mathcal{X}^2$ indicate a Gaussian, a Beta and a Chi-Square distribution, respectively.}
% \label{fig_NDDevaluate}
% \end{figure}
\begin{figure*}[t] 
\centering
\begin{subfigure}{}
    \includegraphics[width=0.25\textwidth, height=0.125\textheight]{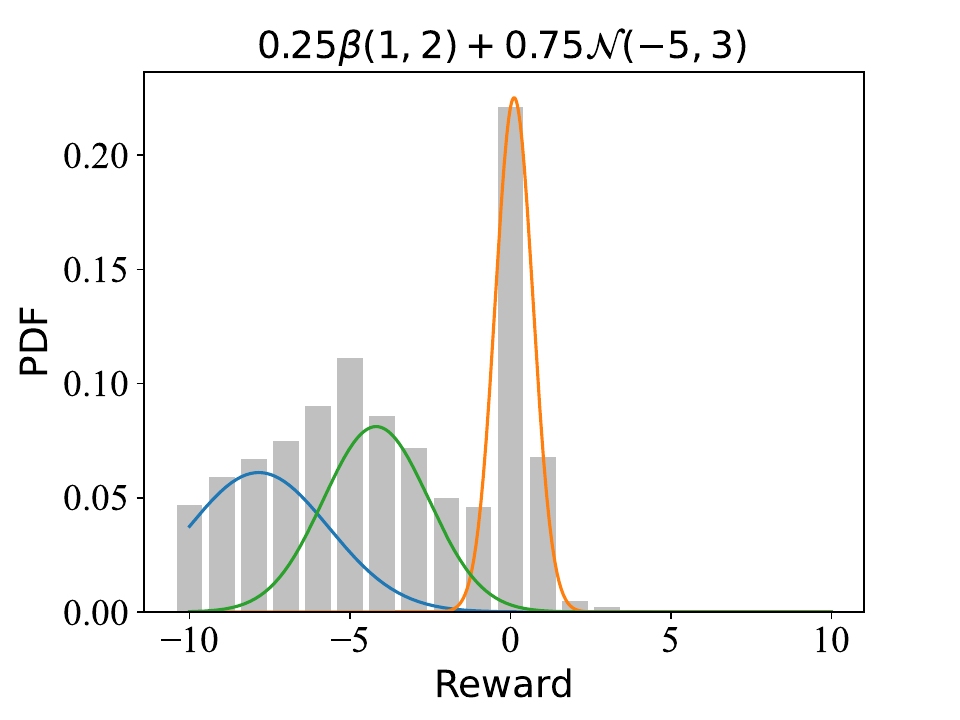}%height=0.15\textheight
    \label{fig_NDDevaluate_0}
\end{subfigure}
\begin{subfigure}{}
    \includegraphics[width=0.25\textwidth, height=0.125\textheight]{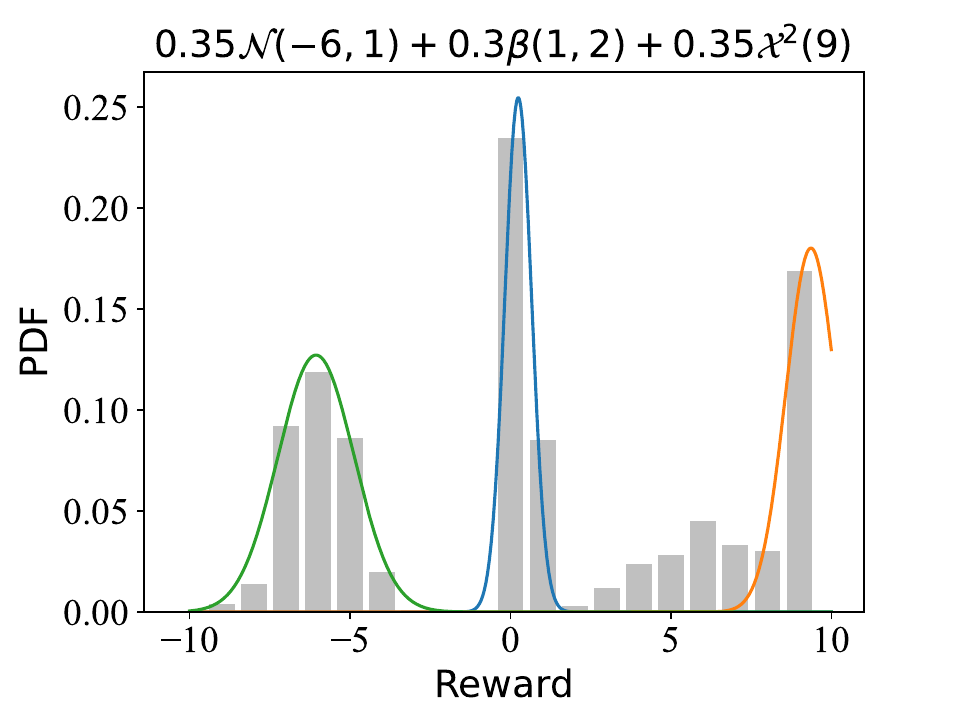}
    \label{fig_NDDevaluate_1}
\end{subfigure}
\begin{subfigure}{}
    \includegraphics[width=0.25\textwidth, height=0.125\textheight]{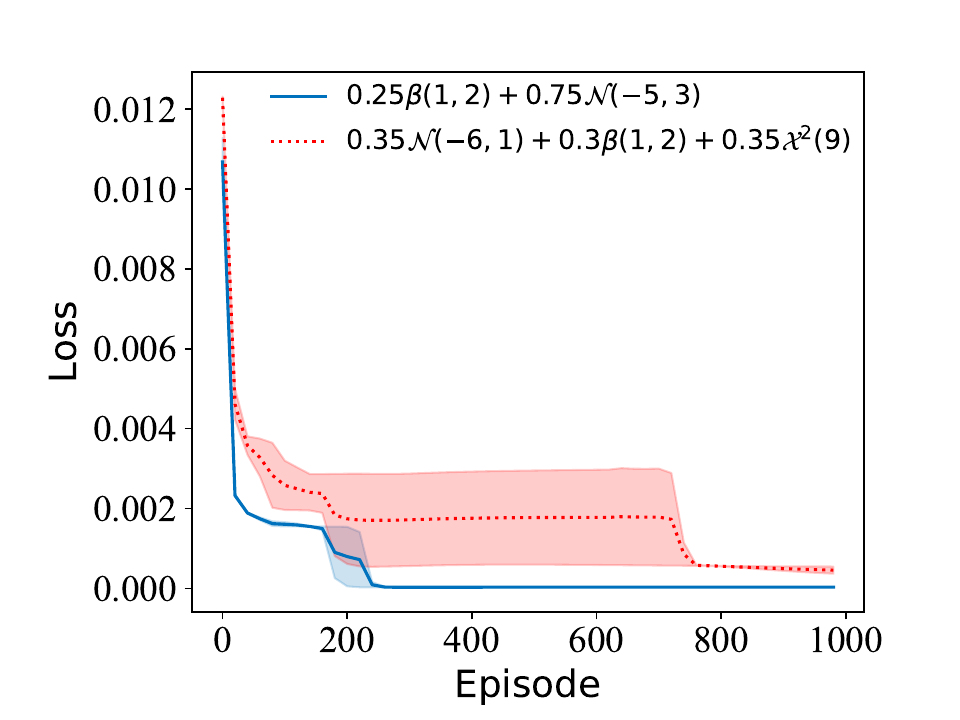}
    \label{fig_NDDevaluate_2}
\end{subfigure}
\caption{Results of distribution decomposition over two noise cases (i.e., $0.25\beta(1,2) + 0.75\mathcal{N}(-5,3)$, $0.35\mathcal{N}(-6,1)+0.3\beta(1,2)+0.35\chi(9)$). The left two sub-figures show the comparison between practical PDF histogram and weighted PDFs of decomposed distributions ($3$ colored curves); while the right sub-figure shows the curve of loss over training iterations. $\mathcal{N}, \beta, \mathcal{X}^2$ indicate a Gaussian, a Beta and a Chi-Square distribution, respectively.}
\label{fig_NDDevaluate}
\end{figure*}
\subsection{Distribution Decomposition Validation}
We further evaluate the performance of NDD under simulated noise with arbitrary distribution. In order to further evaluate the accuracy of NDD, we decompose each noise distribution into $3$ parts and summarize the related intermediate results (i.e., decomposed distributions, corresponding weights, and computed Wasserstein distance $d_p$ between the real and approximated distributions) in Table \ref{tab1}.
Though it generally leads to inferior decomposition accuracy for more complex noise distribution, all the losses are rather low in terms of Wasserstein distance (even for noise including non-Gaussian components). 
To make the decomposition results more intuitive, Fig. \ref{fig_NDDevaluate} compares the PDFs of two noise distributions (i.e., $0.25\beta(1,2)+0.75\mathcal{N}(-5,3)$ and $0.35\mathcal{N}(-6,1)+0.3\beta(1,2)+0.35\chi(9)$) with weighted decomposed distributions, and provides the learning loss curves. It can be found from Fig. \ref{fig_NDDevaluate} that weighted distributions align closely with the noise distribution. Besides, after several hundred episodes of training, it quickly converges.
In other words, the adoption of GMM to characterize the noise distribution is promising.

\begin{table*}[htbp]
\footnotesize
\caption{Average Scores in MPE's Tasks under Different Distortion Risk Functions.}
\label{main_res_remap}
\begin{center}
\setlength{\tabcolsep}{.5mm}{}%0.95mm
\begin{tabular}{|c|c|c|c|c|c|c|c|c|}
\toprule
\textbf{Scenario} & \textbf{Type}&\textbf{CPW 0.71}&\textbf{WANG 0.75}&\textbf{WANG {-}0.75}&\textbf{POW {-}2.0}&\textbf{CVaR 0.25}&\textbf{CVaR 0.1}&\textbf{Expectation}\\
\midrule
\multirow{3}{*}{Adversary}&Noise $0$&$\bm{{-}0.67}$ $\pm$ $0.77$&${-}11.21$$\pm$$7.08$&${-}2.51$$\pm$$0.63$&${-}4.05$$\pm$$0.80$&${-}1.26$$\pm$$\bm{0.42}$&${-}4.89$$\pm$$1.52$&${-}1.72$$\pm$$0.45$\\
&Noise $1$&${-}1.47$$\pm$$0.63$&${-}1.70$$\pm$$0.61$&${-}1.89$$\pm$$1.06$&${-}1.30$$\pm$$1.89$&$\bm{{-}0.15}$$\pm$$\bm{0.35}$&${-}2.85$$\pm$$1.16$&${-}1.93$$\pm$$0.43$\\
&Noise $2$&${-}0.86$$\pm$$0.67$&${-}2.23$$\pm$$0.95$&${-}1.31$$\pm$$0.61$&${-}2.61$$\pm$$0.63$&$\bm{{-}0.78}$$\pm$$0.62$&${-}2.98$$\pm$$0.64$&${-}1.35$$\pm$$\bm{0.40}$\\
\hline
\multirow{3}{*}{Crypto}&Noise $0$&$\bm{{-}8.77}$$\pm$$6.29$&${-}14.74$$\pm$$0.56$&${-}13.85$$\pm$$1.85$&${-}17.15$$\pm$$\bm{0.50}$&${-}9.32$$\pm$$1.20$&${-}14.54$$\pm$$1.68$&${-}12.09$$\pm$$1.69$\\
&Noise $1$&${-}8.86$$\pm$$2.06$&${-}14.75$$\pm$$\bm{0.35}$&${-}9.88$$\pm$$3.79$&${-}10.98$$\pm$$1.19$&${-}10.25$$\pm$$0.91$&${-}11.46$$\pm$$2.73$&$\bm{{-}5.98}$$\pm$$0.81$\\
&Noise $2$&${-}11.19$$\pm$$1.83$&$\bm{{-}8.96}$$\pm$$1.58$&${-}13.70$$\pm$$1.90$&${-}14.74$$\pm$$\bm{0.49}$&${-}10.93$$\pm$$0.64$&${-}9.89$$\pm$$2.54$&${-}10.74$$\pm$$0.52$\\
\hline
\multirow{3}{1cm}{Speaker Listener}&Noise $0$&$\bm{{-}29.70}$$\pm$$1.62$&${-}32.55$$\pm$$2.80$&${-}35.37$$\pm$$\bm{1.50}$&${-}33.02$$\pm$$1.96$&${-}31.88$$\pm$$3.57$&${-}34.65$$\pm$$2.25$&${-}31.37$$\pm$$1.80$\\
&Noise $1$&${-}31.35$$\pm$$1.42$&${-}34.00$$\pm$$3.25$&${-}32.60$$\pm$$1.81$&${-}34.84$$\pm$$2.35$&${-}32.00$$\pm$$1.92$&${-}36.23$$\pm$$1.97$&$\bm{{-}30.12}$$\pm$$\bm{1.34}$\\
&Noise $2$&${-}30.72$$\pm$$1.78$&$\bm{{-}30.26}$$\pm$$1.31$&${-}33.29$$\pm$$1.84$&${-}31.68$$\pm$$1.44$&${-}30.58$$\pm$$\bm{0.96}$&${-}33.67$$\pm$$1.04$&${-}32.00$$\pm$$1.35$\\
\hline
\multirow{3}{*}{Spread}&Noise $0$&${-}94.25$$\pm$$1.37$&${-}106.15$$\pm$$2.80$&${-}97.14$$\pm$$2.15$&${-}101.67$$\pm$$3.07$&${-}94.66$$\pm$$1.54$&${-}96.51$$\pm$$1.62$&$\bm{{-}93.02}$$\pm$$\bm{0.71}$\\
&Noise $1$&${-}94.55$$\pm$$1.29$&${-}102.21$$\pm$$4.33$&${-}99.08$$\pm$$1.27$&${-}99.94$$\pm$$1.78$&${-}95.46$$\pm$$\bm{0.83}$&${-}94.09$$\pm$$2.19$&$\bm{{-}93.36}$$\pm$$1.88$\\
&Noise $2$&${-}93.78$$\pm$$1.98$&${-}98.58$$\pm$$1.81$&${-}99.70$$\pm$$2.08$&${-}99.99$$\pm$$\bm{0.75}$&${-}92.61$$\pm$$1.64$&${-}96.53$$\pm$$1.70$&$\bm{{-}91.24}$$\pm$$1.58$\\
\hline
\multirow{3}{*}{Reference}&Noise $0$&${-}36.46$$\pm$$\bm{0.49}$&${-}37.35$$\pm$$1.40$&$\bm{{-}35.31}$$\pm$$1.19$&${-}35.94$$\pm$$1.15$&${-}37.11$$\pm$$0.96$&${-}39.96$$\pm$$1.00$&${-}37.22$$\pm$$0.89$\\
&Noise $1$&$\bm{{-}35.48}$$\pm$$1.03$&${-}36.70$$\pm$$1.11$&${-}37.36$$\pm$$1.00$&${-}36.71$$\pm$$\bm{0.76}$&${-}36.45$$\pm$$2.04$&${-}37.64$$\pm$$1.23$&${-}36.87$$\pm$$0.79$\\
&Noise $2$&$\bm{{-}36.64}$$\pm$$0.93$&${-}37.05$$\pm$$1.16$&${-}37.25$$\pm$$0.90$&${-}37.02$$\pm$$0.84$&${-}36.86$$\pm$$\bm{0.79}$&${-}39.10$$\pm$$1.17$&${-}37.31$$\pm$$1.33$\\
\bottomrule
\end{tabular}
\end{center}
\end{table*}
\subsection{Results for Different Distortion Risk Functions}

In order to manifest the effectiveness and stability of MARL with distortion risk functions $\rho_{\eta}(\tau)$, we test some formats of $\rho_{\eta}(\tau)$ (i.e., CPW\cite{CPW, CPW071}, WANG\cite{WANG}, POW\cite{dabney2018implicit}, and CVaR\cite{CVaR}) with parameters configured as in Table \ref{riskfunnotations} for NDD. Table \ref{main_res_remap} shows the performance of adopting different $\rho_{\eta}(\tau)$ in five MPE tasks, in terms of the means and standard deviations of the results. Note that means and standard deviations are compared independently. The ``Expectation'' in this Table means NDD,  which uses \emph{expectation} of action value distribution to update policies directly without any $\rho_{\eta}(\tau)$.
% traditional MARL which uses  $Q$-value function to update policies.

It can be observed from Table \ref{main_res_remap} that apart from Spread, wherein the Expectation method only achieves optimal performance, most of the tasks impose superior performance after imposing the risk preference on strategies. The optimal values are distributed relatively evenly among different types and parameters of $\rho_{\eta}(\tau)$, which reflects the significance of distributional MARL.
As a function that better aligns with human decision-making habits, $\text{CPW(0.71)}$  achieves the best results in one-third of the cases.
The stability of strategy with $\text{CVaR(0.25)}$ is the highest, at the cost of poorer performance on the effectiveness.
Interestingly, some distortion risk functions are ineffective. For instance, $\text{CVaR(0.1)}$ does not work in any of the $15$ cases.
In a word, $\text{CPW(0.71)}$ can contribute to maximizing returns. Meanwhile, if stability and fault tolerance are the primary prerequisites, $\text{POW(-2)}$ and $\text{CVaR(0.25)}$ are worthy of the recommendation.

\begin{figure}
    \centerline{\includegraphics[width=3.15in, height=0.2\textheight]{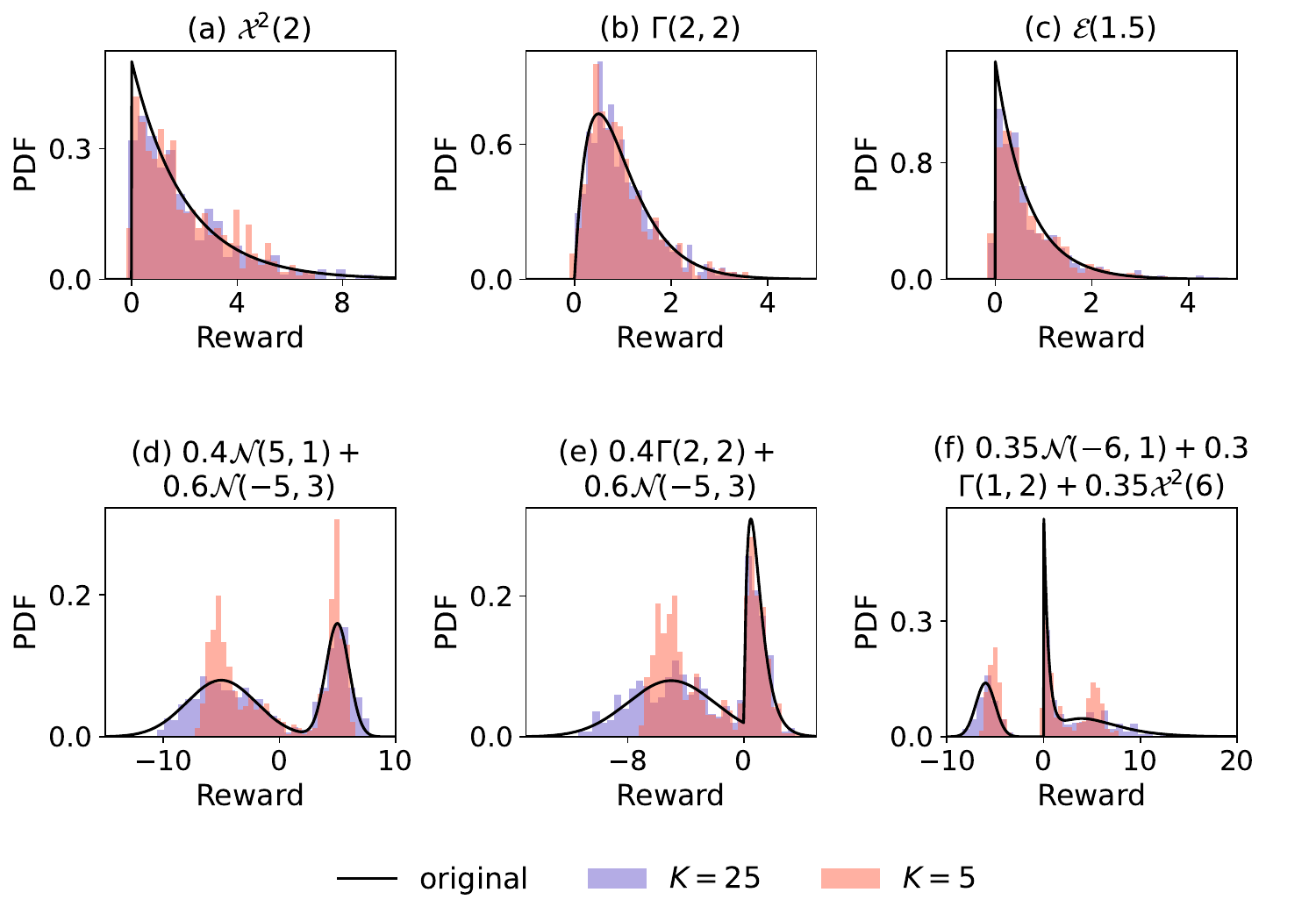}}
    \caption{The performance of DM in generated samples is illustrated. $\mathcal{N}, \mathcal{X}^2, \Gamma, \mathcal{E}$ are a Gaussian, a Chi-Square, a Gamma, and an exponential distribution, respectively.}
    \label{DG_fig}
\end{figure}

\begin{figure*}
    \centerline{\includegraphics[width=.895\textwidth, height=0.325\textheight]{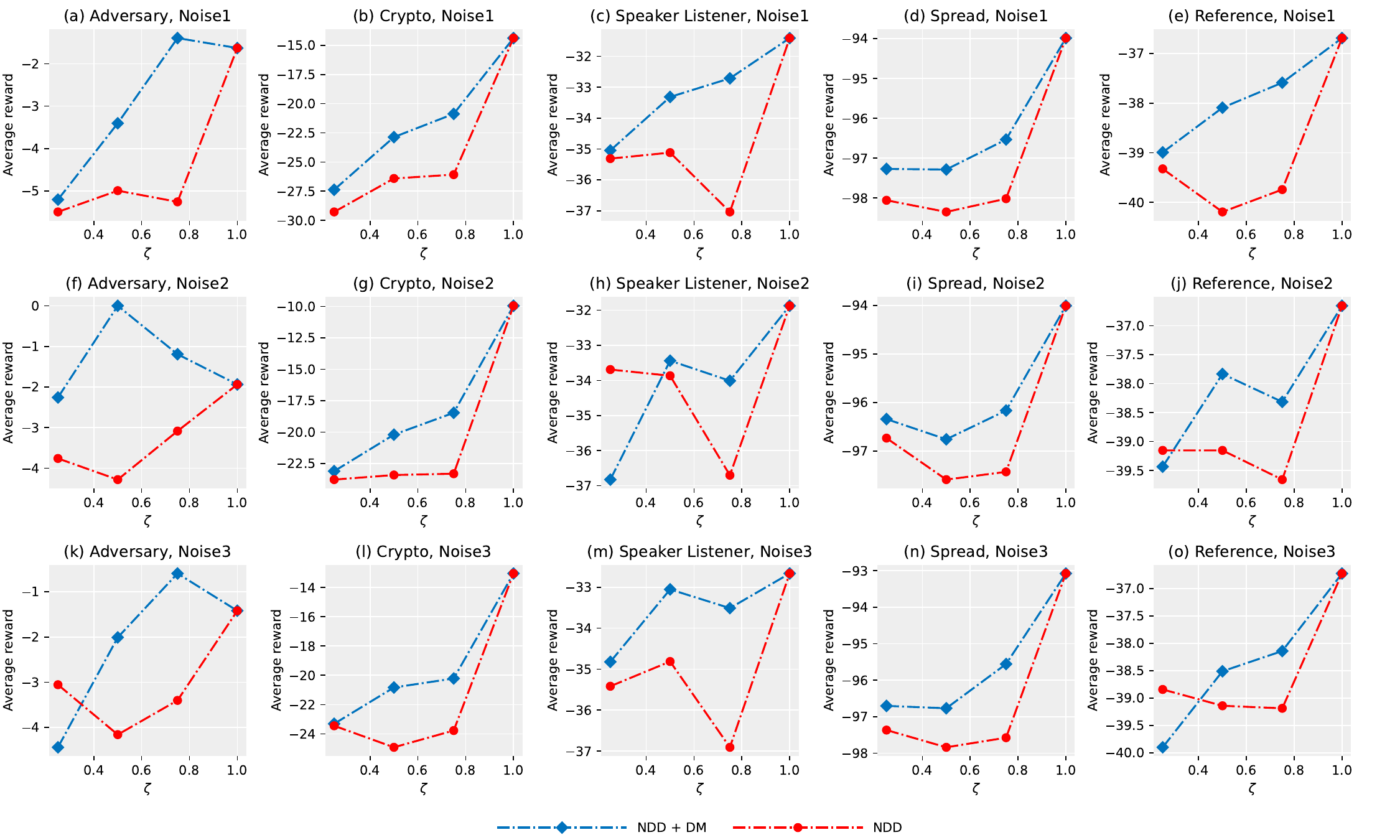}}%
    \caption{The performance comparison of the NDD algorithm before and after using DM.}
    \label{dm_res}
\end{figure*}
\subsection{DM-based Data Augmentation}\label{DM_experiment}

In this section, our experiments mainly focus on assessing the data preciseness provided by DM for NDD. 
Beforehand, we first evaluate the distribution of the data generated by DM and compare it with the practical distribution in Fig. \ref{DG_fig}. In the case with $K=25$, benefiting from the alignment capability of DM, the simulated data can resemble the original data, thus reducing the sampling and communication costs.

Next, Fig. \ref{dm_res} depicts the performance of the NDD approach with and without DM-based data augment. In this part, we intentionally hide some proportion of data for DM to supplement. For simplicity, the ratio of the remaining data quantity to the complete one is denoted as $\zeta$.  Generally speaking, along with the increase of $\zeta$, the NDD approach gives superior performance, consistent with our intuition.
More importantly, the use of DM significantly improves performance, strongly demonstrating the effectiveness of DM. \blut{That's to say, the same training effect can be achieved with significantly fewer training data than that without DM. }
On the other hand, for an over-small $\zeta$, the incorporation of DM leads to a decline in the performance of NDD, as too little data induces a significant gap between the distribution learned by DM and the original distribution. Specifically, according to analyses based on \eqref{eq_theorm_error_exp_upper_bound}, the estimated expectations $\mathbb{E}(r)$ and variances $\mathbb{D}(r)$ are distorted in this data-lacking case, with non-negligible generation error. 
Therefore \blus{Fig. \ref{dm_res} also provides a good reference for balancing performance and interaction cost. Generally speaking, supplementing data within $50\%$ from DM sounds reasonable.}

\section{Conclusion}\label{sectionCon}
In this paper, we propose NDD to tackle the noisy rewards during cooperative and partially observed MARL tasks. In particular, we leverage the GMM-based reward distribution decomposition to characterize and approximate the distributional global noisy rewards by local Gaussian components, so as to mitigate the negative impact of noises. We also introduce distortion risk functions to leverage additional information learned from the action value distribution and discuss their applicability. Additionally, after providing the bounded generation and approximation error, we use diffusion models for data augmentation to address the high training cost issue. Moreover, we theoretically establish the consistency between the globally optimal action and locally optimal ones and carefully calibrate the loss function designed to update NDD networks. Extensive experimental results in MPE and SMAC also prove the effectiveness and superiority of NDD.

\blut{Our future work is dedicated to analyzing the time and space complexity of NDD, as well as exploring the mathematical relationship between hyperparameters in DM and two errors of generation and approximation.}

\bibliographystyle{IEEEtran}
\bibliography{reference}
% Biography text here without a photo.

\begin{IEEEbiography}[{\includegraphics[width=1in,height=1.25in,clip,keepaspectratio]{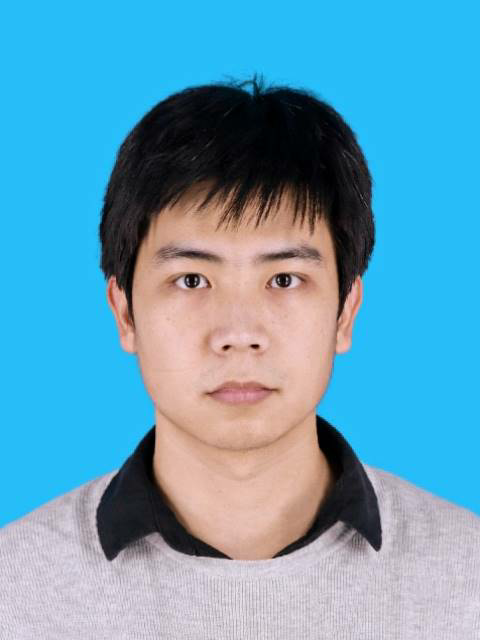}}]{Wei Geng}
He graduated with an M.S. degree in software engineering from the Hangzhou Institute for Advanced Study, UCAS in 2024. Now, he is an intern researcher with Zhejiang Lab. 
His research interests include multi-agent reinforcement learning.
\end{IEEEbiography}
\begin{IEEEbiography}[{\includegraphics[width=1in,height=1.25in,clip,keepaspectratio]{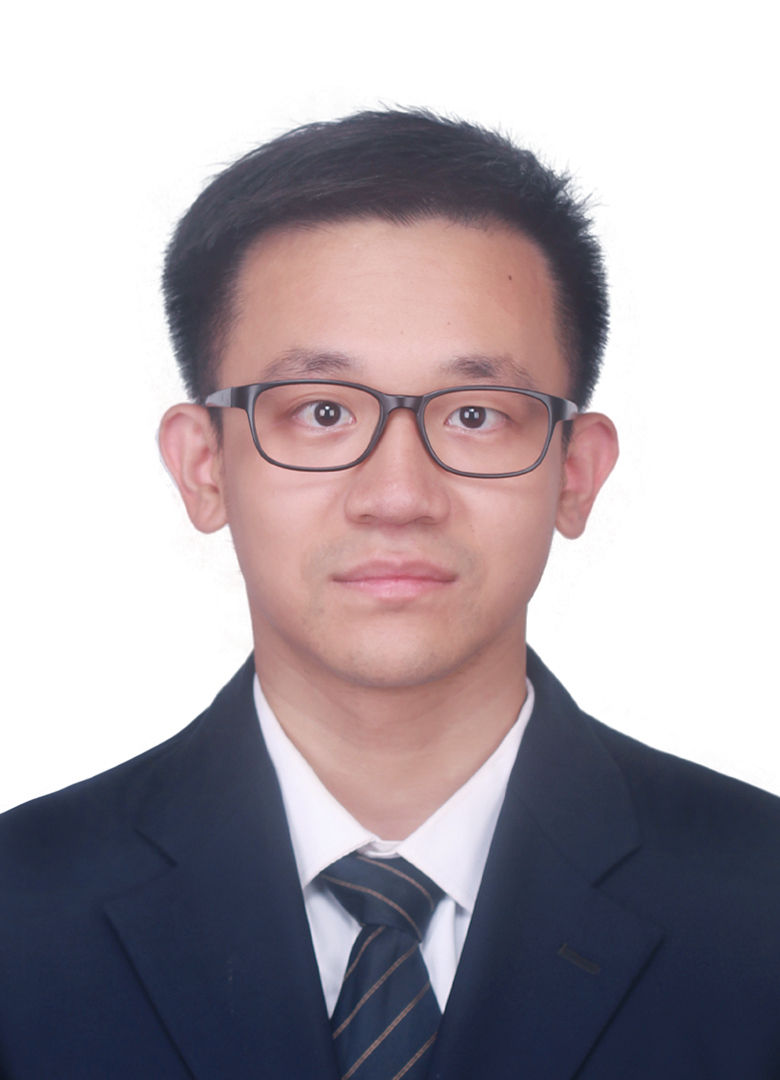}}]{Baidi Xiao}
(Graduate Student Member, IEEE) received the M.S. degree from the College of Information
Science and Electronic Engineering, Zhejiang University, Hangzhou, China. His research interests currently focus on deep learning, multi-agent reinforcement learning, and the Internet of Vehicles.
\end{IEEEbiography}
\begin{IEEEbiography}[{\includegraphics[width=1in,height=1.25in,clip,keepaspectratio]{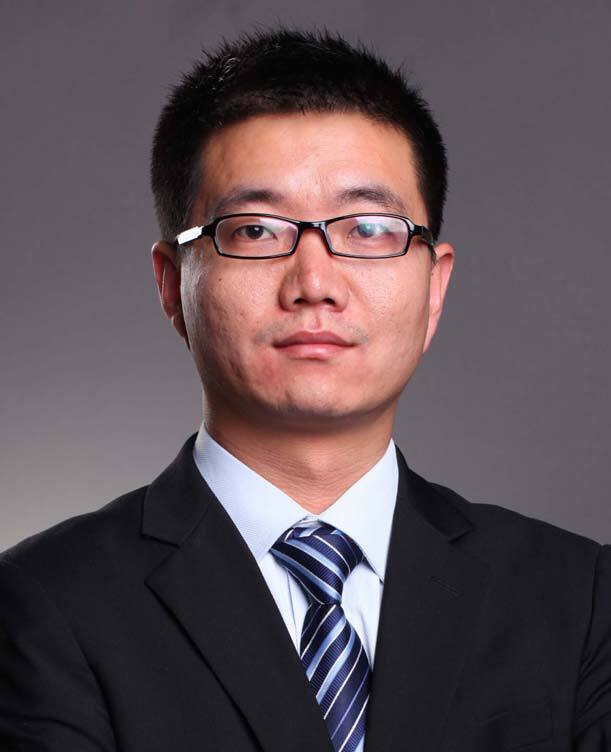}}]{Rongpeng Li}
(Senior Member, IEEE) is currently an Associate Professor with the College of Information Science and Electronic Engineering, at Zhejiang University, Hangzhou, China. He was a Research Engineer with the Wireless Communication Laboratory, Huawei Technologies Company, Ltd., Shanghai, China, from August 2015 to September 2016. He was a Visiting Scholar with the Department of Computer Science and Technology, University of Cambridge, Cambridge, U.K., from February 2020 to August 2020. His research interest currently focuses on networked intelligence for communications evolving (NICE). He received the Wu Wenjun Artificial Intelligence Excellent Youth Award in 2021. He serves as an Editor for \emph{China Communications}.
\end{IEEEbiography}
\begin{IEEEbiography}[{\includegraphics[width=1in,height=1.25in,clip,keepaspectratio]{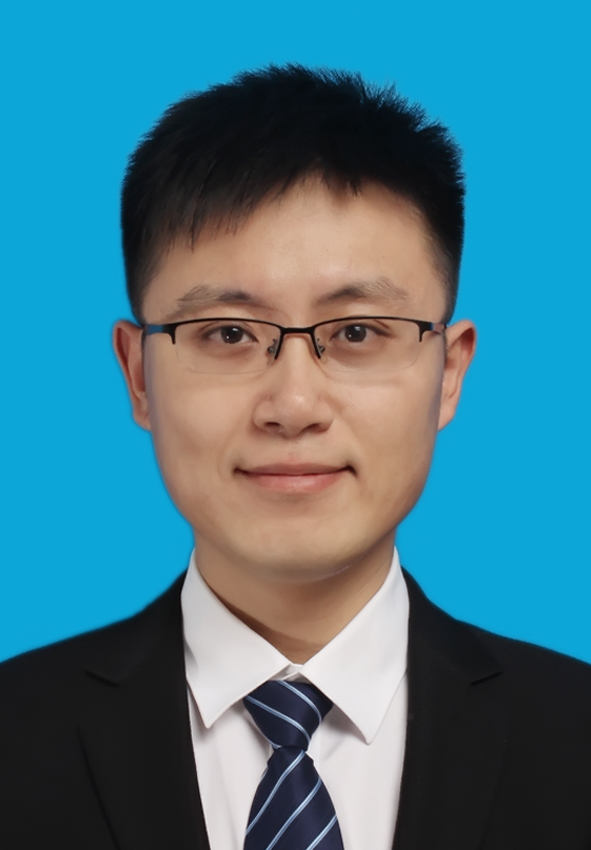}}]{Ning Wei} is currently an Engineering Expert with Zhejiang Lab, Hangzhou, China. He was an AI$\&$Robotic Engineer with KUKA Robotics, Shanghai, China, from May 2014 to November 2017. He was an AI Researcher with the Hikvision Research Institute, Hangzhou, China, from December 2017 to July 2022. His research interest currently focuses on Reinforcement Learning and Simulation$\&$Data Mixed-Driven Intelligence. He is also a member of the China Computer Federation. 
\end{IEEEbiography}
\begin{IEEEbiography}[{\includegraphics[width=1in,height=1.25in,clip,keepaspectratio]{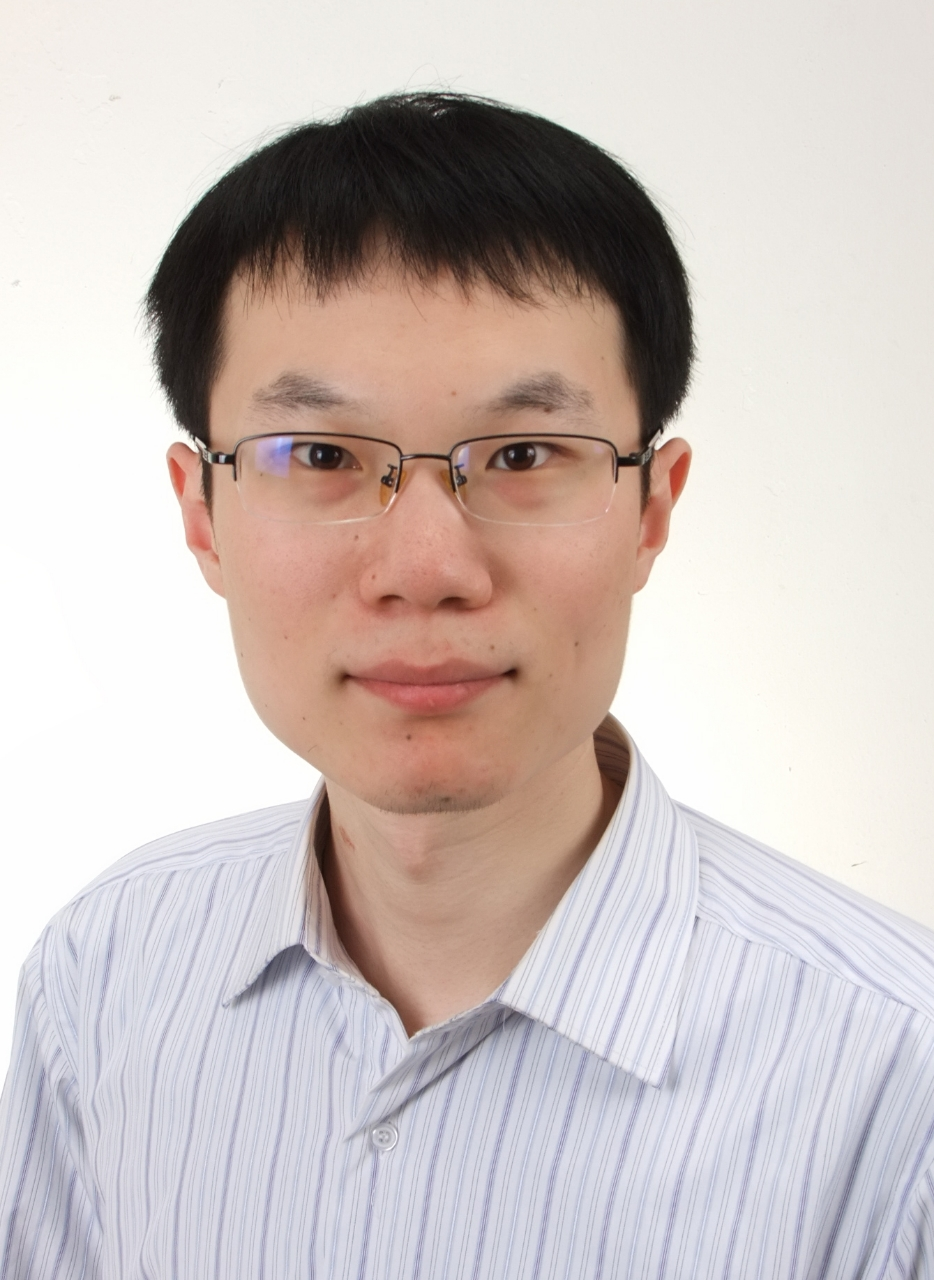}}]{Dong Wang} received the Ph.D. degree in computing from the University of the West of Scotland, Paisley, U.K., in 2020.
He is a Senior Research Engineer at 6G Research Center, China Telecom Research Institute, Beijing, China, and leads the network architecture and services team. He also serves as a TSC co-chair of ONAP, Linux Foundation Networking.
His research interests include 6G mobile networks, network architecture and operation, intent-driven networks, and complex heterogeneous networks.
\end{IEEEbiography}
\begin{IEEEbiography}[{\includegraphics[width=1in,height=1.25in,clip,keepaspectratio]{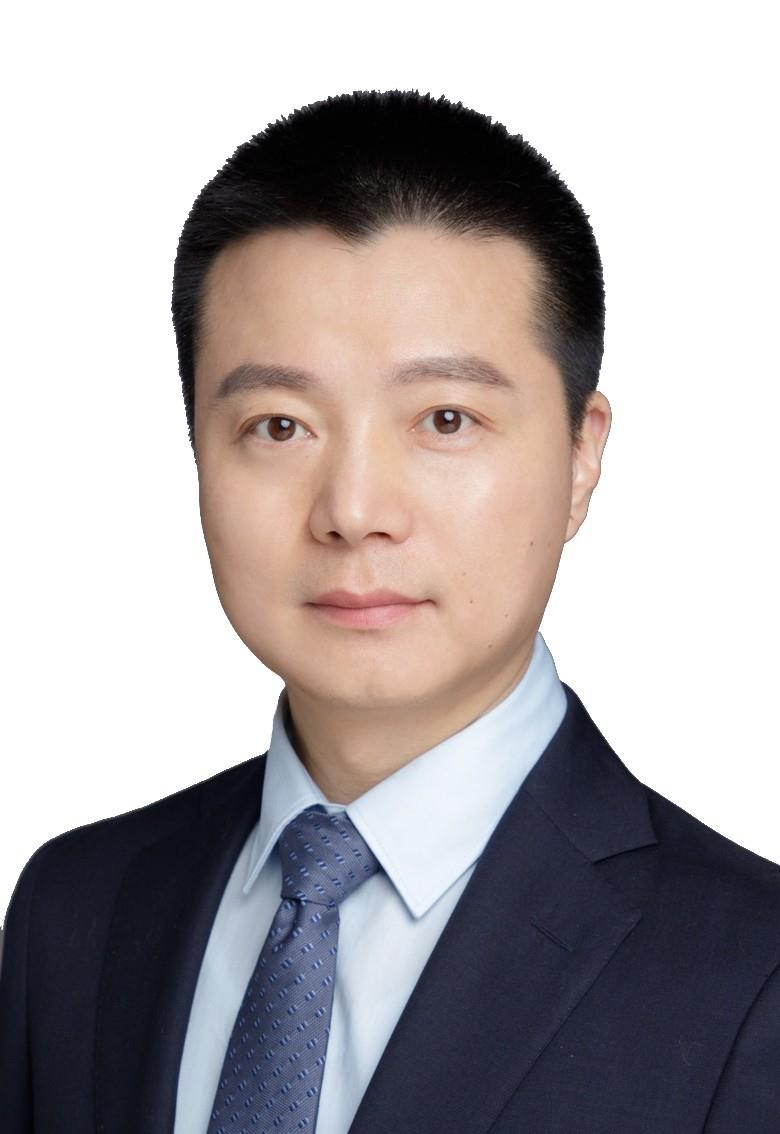}}]{Zhifeng Zhao}
(Member, IEEE) received the B.E. degree in computer science, the M.E. degree in communication and information systems, and the Ph.D. degree in communication and information systems from the PLA University of Science and Technology, Nanjing, China, in 1996, 1999, and 2002, respectively. From 2002 to 2004, he acted as a Post-Doctoral Researcher with Zhejiang University, Hangzhou, China, where his researches were focused on multimedia next-generation networks (NGNs) and soft switch technology for energy efficiency. From 2005 to 2006, he acted as a Senior Researcher with the PLA University of Science and Technology, where he performed research and development on advanced energy-efficient wireless routers, \emph{ad-hoc} network simulators, and cognitive mesh networking test-bed.
From 2006 to 2019, he was an Associate Professor at the College of Information Science and Electronic Engineering, Zhejiang University. Currently, he is with the Zhejiang Lab, Hangzhou as the Chief Engineering Officer. His research areas include software-defined networks (SDNs), wireless networks in 6G, computing networks, and collective intelligence. He is the Symposium Co-Chair of ChinaCom 2009 and 2010. He is the Technical Program Committee (TPC) Co-Chair of the 10th IEEE International Symposium on Communication and Information Technology (ISCIT 2010).
\end{IEEEbiography}

\clearpage
\input{appendix}
\end{document}

%% file: appendix.tex
\appendices 
{
\section{Typical Examples of Distortion Risk Functions}\label{drf_exams}
The action corresponding to the maximum expectation like the action value does not always mean the optimal decision under particular scenarios. Therefore, different risk strategies, which enhance the diversity of policies on account of the distributions over returns \cite{dabney2018implicit}, apply to different task scenarios using distributional RL and often reflect the underlying uncertainties associated with certain actions. Distortion risk functions $\rho_{\eta}(\tau)$ help MAS obtain these risk strategies by remapping action value distribution quantiles, which is particularly important in noisy environments.

% \noindent 
Based on $\rho_{\eta}(\tau)$, four functions, which consider human decision-making habits and optimization costs and include CPW\cite{CPW, CPW071}, WANG\cite{WANG}, POW\cite{dabney2018implicit}, and CVaR\cite{CVaR}, are further presented as follows: 

\begin{itemize}
\item[(1)]
\begin{Eequation}
    \rho_{\eta}(\tau)=\text{CPW}(\eta, \tau)=\frac{\tau^{\eta}}{\big(\tau^{\eta}+(1-\tau)^{\eta}\big)^{\frac{1}{\eta}}}.
\end{Eequation}

The $\text{CPW}$ is the cumulative probability weighting parameterization proposed in cumulative prospect theory \cite{CPW}. It can be observed $\eta$ is locally concave for small values of $\tau$, while convex for larger values of $\tau$.
Concavity corresponds to risk-averse and convexity to risk-seeking policies \cite{CPW071}. \cite{CPW071} recommends $\eta = 0.71$.

\item[(2)]
\begin{Eequation}
    \rho_{\eta}(\tau)=\text{WANG}(\eta, \tau)=\Phi\big(\Phi^{-1}(\tau)+\eta\big).
\end{Eequation}

$\text{WANG}$ is proposed in \cite{WANG} where $\Phi$ and $\Phi^{-1}$ are taken to be the standard CDF of Normal distribution and its inverse.  Intuitively, it produces risk-averse policies when $\eta < 0$ and vice versa.

\item[(3)]
\begin{Eequation}
    \rho_{\eta}(\tau)=\text{POW}(\eta,\tau)=\begin{cases}  \tau^{\frac{1}{1+|\eta|}}, & \text{if } \eta \geq 0; \\  1-(1-\tau)^{\frac{1}{1+|\eta|}}, & \text{otherwise.} \end{cases}
\end{Eequation}

$\text{POW}$ \cite{dabney2018implicit} produces risk-averse effect with $\eta < 0$ or risk-seeking with $\eta \geq 0$.

\item[(4)]
\begin{Eequation}
    \rho_{\eta}(\tau)=\text{CVaR}(\eta,\tau)=\eta\tau.
\end{Eequation}

$\text{CVaR}$ \cite{CVaR} remaps $\tau$ from $[0, 1]$ to $[0, \eta]$, which is widely used in distributional RL.

\end{itemize}

\section{Diffusion Model Analysis}\label{dm_analysis}
DM introduces randomness by adding Gaussian noise $K$ times in the forward diffusion process and then obtains new samples through denoising in the reverse diffusion process \cite{DDPM}. 
% With its adept distribution representation ability, DM offsets the significant sample-collection cost of RL to an extent. 
Specifically, DM starts with $r \sim \mathcal{D}$ and forwards the addition of Gaussian noise. In forward process, $r^{<k>}$ means the generated sample in $k$-th step ($k=1,\cdots,K$) and we have $r^{<k>}\sim \mathcal{D}_q(r^{<k>})$. By the way, $\mathcal{D}_q(r) = \mathcal{D}$.
$\omega_k$ is a pre-defined variance schedule in the $k$-th step with $\upsilon_k = 1 - \omega_k$.
For the forward diffusion chain of DM which adds Gaussian noise to $r$ recursively \cite{pmlr-v37-sohl-dickstein15}, we have for $\omega_1 < \omega_2 < \cdots <\omega_K$,
\begin{equation}
    \mathcal{D}_q(r^{<k>}|r^{<k-1>})=\mathcal{N}(r^{<k>};\sqrt{1-\omega_k}r^{<k-1>},\omega_k),\\
\end{equation}
and
\begin{equation}
    \mathcal{D}_q(r^{<1:K>}|r):=\prod\nolimits_{k=1}^{K} \mathcal{D}_q(r^{<k>}|r^{<k-1>}).
\end{equation}
Unrolling the above recursion, the general term expression for $r^{<k>}$ is as follows \cite{DDPM}:
\begin{equation}
\label{DM_add_noise}    
\mathcal{D}_q(r^{<k>}|r)=\mathcal{N}\big(r^{<k>};\sqrt{\bar{\upsilon}_k}r,(1-\bar{\upsilon}_k)\big),
\end{equation}
where $\bar{\upsilon}_k = \prod_{i=1}^{k} \upsilon_i$.
A reparameterization trick\cite{DDPM} allows for the separation of a random variable (i.e., standard Gaussian distribution) and the mean-variance parameters of a Gaussian distribution, with the latter being inferred by DNNs. Consequently, \eqref{DM_add_noise} can be rewritten as \cite{DDPM}
\begin{equation} \label{dm_general_term}
    r^{<k>} = \sqrt{\bar{\upsilon}_k}r + \sqrt{1 - \bar{\upsilon}_k} \bar{\epsilon}_k,
\end{equation}
where $\bar{\epsilon}_k \sim \mathcal{N}(0, \textbf{I})$. After $K$ noise additions, $r^{<K>}$ approaches standard Gaussian distribution. 

Reversely, the reverse diffusion process predicts the Gaussian parameters (mean and variance) from $r^{<k+1>}$ to $r^{<k>}$($\forall k \in \{0, \cdots, K-1\}$), until obtaining the generated data $r^{<0>}$. Mathematically, based on $r^{<k>}$ and $r$, we have \cite{DDPM}
\begin{equation}
\begin{aligned}
    &\mathcal{D}_q(r^{<k-1>}|r^{<k>},r)\sim \\
    &\mathcal{N}\big(r^{<k-1>};\frac{1}{\sqrt{\upsilon_k}}(r^{<k>}-\frac{\omega_k}{\sqrt{1-\bar{\upsilon}_k}}\bar{\epsilon}_k),\frac{1-\bar{\upsilon}_{k-1}}{1-\bar{\upsilon}_k}\omega_k\big).
\end{aligned}
\end{equation}

\noindent Notably, at the $k$-th step of the diffusion process, $\bar{\epsilon}_k$ is inferred by a DNN with parameters $\Theta$. Therefore, by the reparameterization trick, the recursive relationship from $r^{<k>}$ to $r^{<k-1>}$ can be expressed as \cite{DDPM}
\begin{align} 
    & r^{<k-1>} \leftarrow \label{dm_reverse_eq} \\
    &\frac{1}{\sqrt{\upsilon_k}}\big(r^{<k>}-\frac{\omega_k}{\sqrt{1-\bar{\upsilon}_k}}{\epsilon_{\Theta}}(r^{<k>}, k)\big) + \sqrt{\frac{1-\bar{\upsilon}_{k-1}}{1-\bar{\upsilon}_k}\omega_k}\epsilon, \nonumber
\end{align}
where ${\epsilon_{\Theta}}(r^{<k>}, k)$ is the approximated value of $\bar{\epsilon}_k$ by DNNs and $\epsilon\sim \mathcal{N}(0,1)$. Thus the trainable reverse diffusion chain can be optimized by the loss function as \eqref{eq:loss_dm} on Page \pageref{eq:loss_dm}. The structure of the reverse diffusion process is in Fig. \ref{DM_fig}, where the recursive layer adopts \eqref{dm_reverse_eq}.
\begin{figure*}[tb]
    \begin{equation}
        \label{eq:loss_dm}
            \mathcal{L}_{\text{DM}}(\Theta):=
            \mathbb{E}_{k\sim\mathcal{U}(1,K),\iota \sim\mathcal{N}(0,1),r\sim \mathcal{D}_q(r) }
            [||\iota-\epsilon_{\Theta}(\sqrt{\bar{\upsilon}_k }r+\sqrt{1- 
             \bar{\upsilon}_k}\iota,k)||^2]
    \end{equation}
    \hrulefill
\end{figure*}

\blus{
With its adept distribution representation ability (e.g. $\mathcal{D}$), DM is used for data augmentation to offset the significant sample-collection cost of NDD to an extent. Specifically, DM-based data augmentation can be divided into the training process and the generation process. The training process is carried out independently over $K$ diffusion steps. In each step, there is an independent neural network used to predict the diffused distribution ${\epsilon_{\Theta}}(r^{<k>}, k)$. The neural network updates itself to minimize \eqref{eq:loss_dm}, which measures the difference between predicted and real distribution. Note that the data augmentation process corresponds to the reverse diffusion process in DM.
The whole algorithm for DM-based data augmentation is shown in Algorithm \ref{dm_train_algorithm}.
}

\begin{algorithm}[t]
\caption{DM-based Data Augmentation}\label{dm_train_algorithm}
\KwIn{Original reward sample $r$, the maximum training step $T$, no. of diffusion steps $K$, predefined hyper-parameter $\upsilon_k$ and $\omega_k, k\in \{1,\cdots,K\}$.}
\KwOut{Generated sample $r^{<0>}$.}
Initialize $\Theta$, $t\leftarrow 0$\;
\While{$t \leq T$}{
Randomly choose $k$ from $\{1,\cdots,K\}$\;
Sample $\iota$ from $\mathcal{N}(0,1)$\;
Update $\Theta$ by minimizing \eqref{eq:loss_dm}\;
}
Sample $r^{<K>}$ from $\mathcal{N}(0,1)$\;
\For{$k = K$ to $1$}{
Infer $\epsilon_{\Theta}$ based on $r^{<k>}$\;
Update $r^{<k-1>}$ by $r^{<k>}$ according to \eqref{dm_reverse_eq}\;
}
\Return $r^{<0>}$.
\end{algorithm}

\begin{figure}
\centerline{\includegraphics[width=4in]{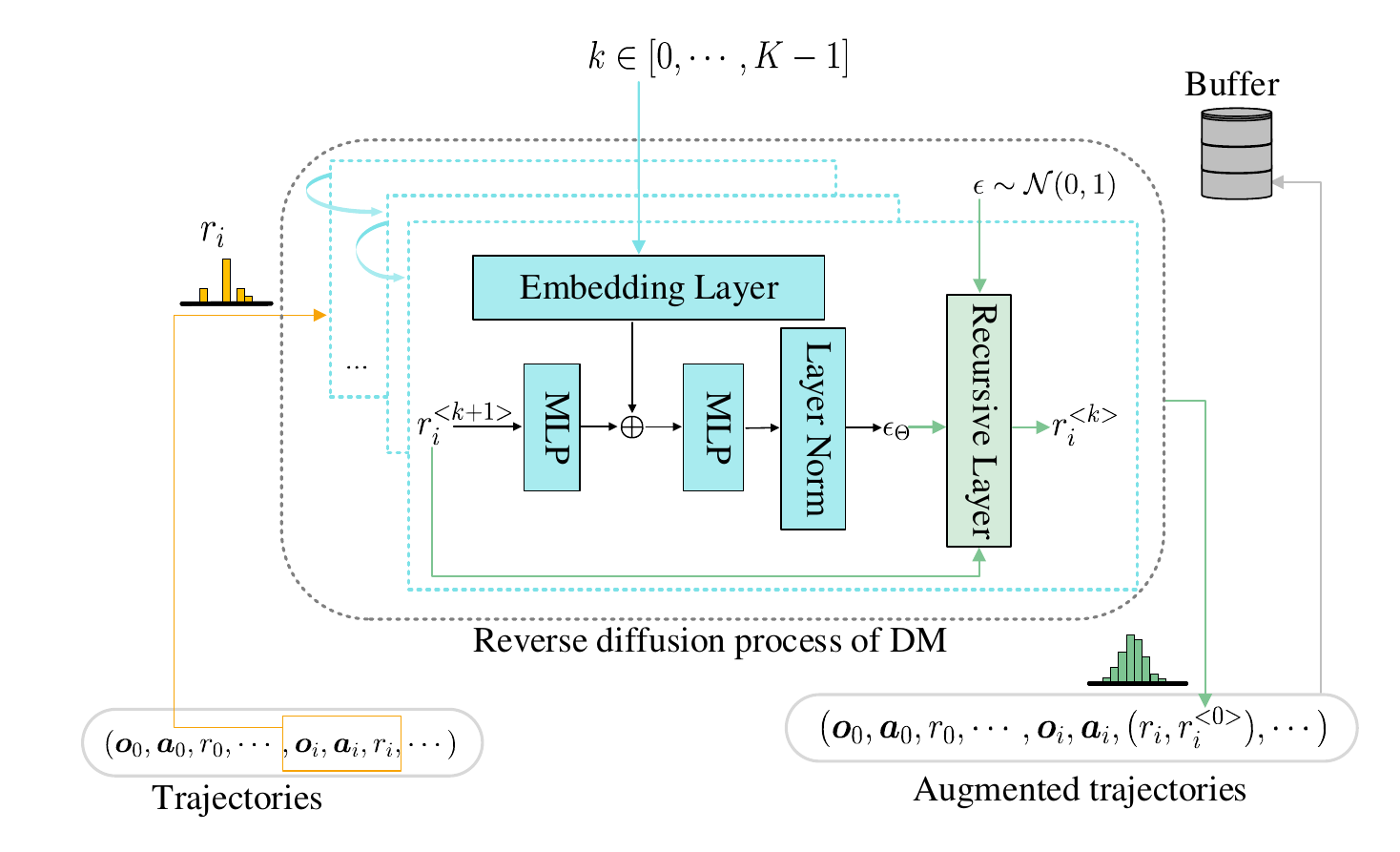}}
\caption{Overview of the DM structure.}
\label{DM_fig}
\end{figure}

\begin{table*}[t]
    \small
    \centering
    \caption{Noise Configuration along Simulation Environments. Here $\mathcal{N},\mathcal{U},\Gamma$, and $\chi^2$ Represent the Gaussian, Uniform, Gamma, and Chi-Square Distribution, Respectively.\protect\tablefootnote{ It's worth noting that Noise $1$ of MPE in Table \ref{tab:noise_config_along_scene} represents that the noise follows $\mathcal{N}(-1,5)$ with a probability of 0.5 and $\mathcal{N}(1,5)$ also with a probability of 0.5, instead of a simple weighted sum. Similar expressions in the following text should be interpreted accordingly.}}
    \begin{tabular}{|c|c|c|c|}
        \toprule
         Type & Noise $0$ & Noise $1$ & Noise $2$\\
        \midrule
        MPE & $\mathcal{N}(0,5)$ & $0.5\mathcal{N}(-1,5)+0.5\mathcal{N}(1,5)$ & $0.3\mathcal{N}(-1,5)+0.2\mathcal{N}(-2,5)+0.5\mathcal{N}(1.4,5)$\\
        \hline
        SMAC & $\mathcal{N}(0,0.1)$ & $0.5\mathcal{N}(-0.1,0.1)+0.5\mathcal{N}(0.1,0.1)$ &  $0.3\mathcal{N}(-0.1,0.1)+0.2\mathcal{N}(-0.2,0.1)+0.5\mathcal{N}(0.14,0.1)$\\
        \midrule

         Type & \multicolumn{2}{c|}{  Noise $3$} &   Noise $4$\\
        \midrule
         MPE & \multicolumn{2}{c|}{ $0.75\beta(2,2)+0.25\mathcal{N}(-1.5, 5)$} &  $0.3\Gamma(1,2)+0.3\mathcal{U}(-12,0)+0.4\chi^2(3)$\\
        \hline
         SMAC & \multicolumn{2}{c|}{ $0.75\beta(0.1,1.9)+0.25\mathcal{N}(-0.15, 0.1)$} & $ 0.3\mathcal{N}(-0.05,0.1)+0.3\mathcal{U}(-0.4,0.1)+0.4\chi^2(0.15)$\\
        \bottomrule

    \end{tabular}
    \label{tab:noise_config_along_scene}
\end{table*}

\begin{table}[]
    \caption{Hyper Parameters in MPE and SMAC.}
    \label{hyper_parameters_tab}
    \centering
    % \footnotesize
    % \small
    \begin{tabular}{l l l}
        \hline
         Hyper Parameter & MPE & SMAC \\
        \hline
        No. of iterations & $3000$ & $2 \times 10^4$\\
        $\lambda$ & $1$ & $1$\\
        $\alpha$ & $1$ & $1$\\
        Buffer size & $1024$ & $2048$\\
        Batch size & $32$ & $64$\\
        $\gamma$ & $0.99$ & $0.99$\\
        Error threshold $e_{\text{min}}$ & $10^{-3}$ & $10^{-3}$\\
        Learning rate & $3 \times 10^{-3}$ & $3 \times 10^{-3}$\\
        \hline
    \end{tabular}
\end{table}

\begin{table}[]
    \caption{Parameters in Distortion Risk Functions.}
    \label{riskfunnotations}
    \centering
    % \footnotesize
    % \small
    \begin{tabular}{l l}
        \hline
         $\rho_{\eta}(\tau)$ & $\eta$ \\
        \hline
        $\text{CPW}(\eta,\tau)$ & $0.71$\\
        $\text{WANG}(\eta,\tau)$ & $0.75,-0.75$\\
        $\text{POW}(\eta,\tau)$ & $-2$\\
        $\text{CVaR}(\eta,\tau)$ & $0.1,0.25$\\
        \hline
    \end{tabular}
\end{table}

\section{Proof of Theorem \ref{them_generation_error}} \label{generation_error}

\begin{proof}
In \eqref{dm_general_term}, $\bar{\epsilon}_k$ is always inferred by $\epsilon_{\Theta}(r^{<k>}, k)$, and thus with the training loss $\xi_1(k)$ for the $k$-th step in DM, $\epsilon_{\Theta}(r^{<k>}, k)$ can be
\begin{align} \nonumber
    \epsilon_{\Theta}(r^{<k>}, k) \nonumber
    =& \bar{\epsilon}_k - \xi_1(k)= \frac{r^{<k>} - \sqrt{\bar{\upsilon}_k}r}{\sqrt{1 - \bar{\upsilon}_k}} - \xi_1(k), \nonumber
\end{align}
where $r$ represents the original sample from the training dataset, which is different from the generated $r^{<0>}$ through the reverse diffusion chain.
And then \eqref{dm_reverse_eq} can be rewritten as \eqref{dm_reverse_r0_rk} where $\epsilon \sim \mathcal{N}(0,1)$ and $k=2,\cdots,K$.
\begin{figure*}[t]
\hrulefill
\begin{equation}
\begin{aligned} \label{dm_reverse_r0_rk}
r^{<k-1>} =& \frac{1}{\sqrt{\upsilon_k}}\Big[r^{<k>}-\frac{1-\upsilon_k}{{1-\bar{\upsilon}_k}}\big(r^{<k>}- \sqrt{\bar{\upsilon}_k}r - {\sqrt{1-\bar{\upsilon}_k}}\xi_1(k)\big)\Big] + \sqrt{\frac{1-\bar{\upsilon}_{k-1}}{1-\bar{\upsilon}_k} \omega_k} \epsilon \\
=& \frac{1}{\sqrt{\upsilon_k}} \frac{\upsilon_k-\bar{\upsilon}_k}{1-\bar{\upsilon}_k} r^{<k>} + \frac{1}{\sqrt{\upsilon_k}} \frac{1-\upsilon_k}{{1-\bar{\upsilon}_k}}\big(\sqrt{\bar{\upsilon}_k}r + {\sqrt{1-\bar{\upsilon}_k}}\xi_1(k)\big) +\sqrt{\frac{1-\bar{\upsilon}_{k-1}}{1-\bar{\upsilon}_k} \omega_k}  \epsilon.
\end{aligned}
\end{equation}  
\hrulefill
\end{figure*}

% \begin{align} \nonumber
%     \Gamma_2(k)=
%     \begin{cases}
%     \frac{\sqrt{1-\upsilon_k}}{\sqrt{\upsilon_k}}\epsilon, & k=1;\\
%     \frac{1}{\sqrt{\upsilon_k}} \frac{1-\upsilon_k}{{1-\bar{\upsilon}_k}}\big(\sqrt{\bar{\upsilon}_k}r + {\sqrt{1-\bar{\upsilon}_k}}\xi_1(k)\big) & \\
%     \quad +\sqrt{\frac{1-\bar{\upsilon}_{k-1}}{1-\bar{\upsilon}_k} \omega_k} \epsilon, &  k=2,\cdots,K,
%     \end{cases}
%     .
% \end{align}

Let
\begin{align} \nonumber
    \Gamma_1(k)=
    \begin{cases}
    \frac{1}{\sqrt{\upsilon_k}}, &\quad k=1;\\
    \frac{1}{\sqrt{\upsilon_k}} \frac{\upsilon_k-\bar{\upsilon}_k}{1-\bar{\upsilon}_k}, &\quad k=2,\cdots,K,
    \end{cases}
\end{align}
and
\begin{align} \nonumber
    \Gamma_2(k)=
    \begin{cases}
    \frac{\sqrt{1-\upsilon_k}}{\sqrt{\upsilon_k}}\epsilon, & k=1;\\
    \frac{1}{\sqrt{\upsilon_k}} \frac{1-\upsilon_k}{{1-\bar{\upsilon}_k}}\big(\sqrt{\bar{\upsilon}_k}r + {\sqrt{1-\bar{\upsilon}_k}}\xi_1(k)\big) & \\
    \quad +\sqrt{\frac{1-\bar{\upsilon}_{k-1}}{1-\bar{\upsilon}_k} \omega_k} \epsilon, &  k=2,\cdots,K.
    \end{cases}
    % .
\end{align}
We have the formula of the general term as
\begin{equation} \label{dm_reverse_general_terms}
    {r}^{<0>}=
    {r}^{<K>}\prod\nolimits_{i=1}^{K}\Gamma_1(i) + \big[\sum\nolimits_{i=1}^{K}\Gamma_2(i) \prod\nolimits_{j=1}^{i-1}\Gamma_1(j)\big]
\end{equation}
Significantly, $\Gamma_1(k) \geq 0, \forall k = 1, \cdots, K$.
Let ${k_{\text{exp\_max}}} := \mathop{\text{argmax}}\nolimits_{k} \Big(\mathbb{E}\big[\Gamma_2(k)\big]\Big)$, 
recalling that $r^{<K>} \sim \mathcal{N}(0,1)$ for a sufficiently large $K$ and $\epsilon \sim \mathcal{N}(0,1)$,
we have \eqref{diff_upper_bound_B}.
\begin{figure*}[t]
\begin{align}  \label{diff_upper_bound_B}
    \mathbb{E}(r^{<0>} - r)
    &\le \frac{1}{\sqrt{\upsilon_{k_{\text{exp\_max}}}}} \frac{1-\upsilon_{k_{\text{exp\_max}}}}{{1-\bar{\upsilon}_{k_{\text{exp\_max}}}}}\left[\sum\nolimits_{i=1}^{K}\prod\nolimits_{j=1}^{i-1}\Gamma_1(j)\right] \cdot \left[\sqrt{\bar{\upsilon}_{k_{\text{exp\_max}}}}\mathbb{E}(r)+{\sqrt{1-\bar{\upsilon}_{k_{\text{exp\_max}}}}}\xi_1({k_{\text{exp\_max}}})\right]- \mathbb{E}(r).
\end{align}
\begin{align} \label{diff_upper_bound_D}
    \mathbb{E}(r^{<0>} - r)
    &\ge \frac{1}{\sqrt{\upsilon_{k_{\text{exp\_min}}}}} \frac{1-\upsilon_{k_{\text{exp\_min}}}}{{1-\bar{\upsilon}_{k_{\text{exp\_min}}}}}\left[\sum\nolimits_{i=1}^{K}\prod\nolimits_{j=1}^{i-1}\Gamma_1(j)\right] \cdot  \left[\sqrt{\bar{\upsilon}_{k_{\text{exp\_min}}}}\mathbb{E}(r)+{\sqrt{1-\bar{\upsilon}_{k_{\text{exp\_min}}}}}\xi_1({k_{\text{exp\_min}}})\right]- \mathbb{E}(r).
\end{align}
\hrulefill
\begin{align} \label{eq_sub_exp_abs}
    & \big|\mathbb{E}(r^{<0>} - r)\big| \nonumber \\
    \le
    & \text{max}\Bigg(\left|\frac{1}{\sqrt{\upsilon_{k_{\text{exp\_max}}}}} \frac{1-\upsilon_{k_{\text{exp\_max}}}}{{1-\bar{\upsilon}_{k_{\text{exp\_max}}}}}\left[\sum\nolimits_{i=1}^{K}\prod\nolimits_{j=1}^{i-1}\Gamma_1(j)\right]\left[\sqrt{\bar{\upsilon}_{k_{\text{exp\_max}}}}+{\sqrt{1-\bar{\upsilon}_{k_{\text{exp\_max}}}}}\xi_{1,\texttt{rel}}({k_{\text{exp\_max}}})\right]- 1\right|, \left|\frac{1}{\sqrt{\upsilon_{k_{\text{exp\_min}}}}} \right. \nonumber \\
    & \left. \frac{1-\upsilon_{k_{\text{exp\_min}}}}{{1-\bar{\upsilon}_{k_{\text{exp\_min}}}}}\left[\sum\nolimits_{i=1}^{K}\prod\nolimits_{j=1}^{i-1}\Gamma_1(j)\right]\left[\sqrt{\bar{\upsilon}_{k_{\text{exp\_min}}}}+{\sqrt{1-\bar{\upsilon}_{k_{\text{exp\_min}}}}}\xi_{1,\texttt{rel}}({k_{\text{exp\_min}}})\right]- 1\right|\Bigg) \cdot \left|\mathbb{E}(r)\right|.
\end{align}
\hrulefill
\begin{equation}
\label{eq_sub_var}
\mathbb{D}(r^{<0>} - r)
\le \left[\prod\nolimits_{i=1}^{K}\Gamma_1(i)\right]^{2} + 
\left[\frac{\bar{\upsilon}_{k_{\text{var\_max}}}}{\upsilon_{k_{\text{var\_max}}}} (\frac{1-\upsilon_{k_{\text{var\_max}}}}{{1-\bar{\upsilon}_{k_{\text{var\_max}}}}})^2 \cdot \mathbb{D}(r)+\frac{1-\bar{\upsilon}_{{k_{\text{var\_max}}}-1}}{1-\bar{\upsilon}_{k_{\text{var\_max}}}} \omega_{k_{\text{var\_max}}}\right]\left[\sum\nolimits_{i=1}^{K}\prod\nolimits_{j=1}^{i-1}\Gamma_1(j)\right]^2 + \mathbb{D}(r).
\end{equation}
\hrulefill
\begin{align}
    & \mathbb{E}\big[(r^{<0>} - r)^2\big] =\big|\mathbb{E}(r^{<0>} - r)\big|^2 + \mathbb{D}(r^{<0>} - r) \nonumber \\
     \le & \text{max}\bigg(\Big\{\frac{1}{\sqrt{\upsilon_{k_{\text{exp\_max}}}}} \frac{1-\upsilon_{k_{\text{exp\_max}}}}{{1-\bar{\upsilon}_{k_{\text{exp\_max}}}}}\big[\sum\nolimits_{i=1}^{K}\prod\nolimits_{j=1}^{i-1}\Gamma_1(j)\big]
    \big[\sqrt{\bar{\upsilon}_{k_{\text{exp\_max}}}}+{\sqrt{1-\bar{\upsilon}_{k_{\text{exp\_max}}}}} \xi_{1,\texttt{rel}}({k_{\text{exp\_max}}})\big]- 1\Big\}^2, \nonumber\\
    & \Big\{\frac{1}{\sqrt{\upsilon_{k_{\text{exp\_min}}}}}
    \frac{1-\upsilon_{k_{\text{exp\_min}}}}{{1-\bar{\upsilon}_{k_{\text{exp\_min}}}}}\big[\sum\nolimits_{i=1}^{K}\prod\nolimits_{j=1}^{i-1}\Gamma_1(j)\big] \big[\sqrt{\bar{\upsilon}_{k_{\text{exp\_min}}}}+{\sqrt{1-\bar{\upsilon}_{k_{\text{exp\_min}}}}}\xi_{1,\texttt{rel}}({k_{\text{exp\_min}}})\big]- 1\Big\}^2\bigg)[\mathbb{E}(r)]^2 + \nonumber \\
    & \left[\prod\nolimits_{i=1}^{K}\Gamma_1(i)\right]^{2} + \left[\frac{\bar{\upsilon}_{k_{\text{var\_max}}}}{\upsilon_{k_{\text{var\_max}}}} (\frac{1-\upsilon_{k_{\text{var\_max}}}}{{1-\bar{\upsilon}_{k_{\text{var\_max}}}}})^2
    \mathbb{D}(r)+\frac{1-\bar{\upsilon}_{{k_{\text{var\_max}}}-1}}{1-\bar{\upsilon}_{k_{\text{var\_max}}}} \omega_{k_{\text{var\_max}}}\right] \left[\sum\nolimits_{i=1}^{K}\prod\nolimits_{j=1}^{i-1}\Gamma_1(j)\right]^2 + \mathbb{D}(r). \label{eq_sum_mean_var}
\end{align}
\hrulefill
\end{figure*}
For the same reason, if let ${k_{\text{exp\_min}}} := \mathop{\text{argmin}}\limits_{k}\Big(\mathbb{E}\big[\Gamma_2(k)\big]\Big)$, we have \eqref{diff_upper_bound_D}. Furthermore, \eqref{eq_sub_exp_abs} holds, where $\xi_{1,\texttt{rel}}(k) = \frac{\xi_{1}(k)}{\mathbb{E}(r)}$ means the relative error in the $k$-th step of the diffusion process.

Similarly for the variance, let $\mathop{\text{argmax}}\limits_{k} \Big(\mathbb{D}\big[\Gamma_2(k)\big]\Big)={k_{\text{var\_max}}}$, we have \eqref{eq_sub_var}. And according to \eqref{eq_sub_exp_abs} and \eqref{eq_sub_var}, \eqref{eq_sum_mean_var} is confirmed.

% \noindent\emph{Remark}:
Based on our settings in Appendix \ref{hyper_parameters_setting}, we have $\sum_{i=1}^{K}\prod_{j=1}^{i-1} \Gamma_1(j) \approx 2.0337$ and $\prod_{i=1}^{K}\Gamma_1(i) \approx 3.5590\times 10^{-4}$, as well as 
the expected error between the generated samples and original samples is as
\begin{align} 
    &\mathbb{E}\big[(r^{<0>} - r)^2\big] \nonumber\\
    &\le \bigg[2.0337 \times \Big(1 + 0.2466\text{max}\big(|\xi_{1,\texttt{rel}}(k)|\big)\Big) -1\bigg]^2\big(\mathbb{E}(r)\big)^2 \nonumber\\
    &\quad + 1.2667\times 10^{-7} + \big[\mathbb{D}(r) + 0.9226 \times 4.9877 \times 10^{-3}\big] \nonumber \\
    &\quad \times 2.0337^2 +\mathbb{D}(r) \label{eq_theorm_error_exp_upper_bound} \\ 
    & \approx \Big[1.0337 + 0.5016\text{max}\big(|\xi_{1,\texttt{rel}}(k)|\big)\Big]^2 \big(\mathbb{E}(r)\big)^2 + \nonumber\\
    &\quad 5.1359\mathbb{D}(r) + 0.01904.\nonumber
\end{align}
It's worth noting that $|\xi_{1,\texttt{rel}}(k)|$ is much less than $1$, and $0.01904$ is much less than the first two terms. Therefore, we can have the conclusion. 
% \blumt{Additionally, it is clear that although DM can decrease interaction cost, too few interactions will lead to inaccurate estimation of expectations and variances, thereby increasing generation error.}
\end{proof}

\section{Proof of Theorem \ref{them_approximation_error}} \label{approximation_error}

Beforehand, we introduce the following useful lemmas.
\begin{lemma} \label{th_gmm_linear}
$Z=aX+bY$ obeys GMM
when given $X \sim \texttt{\large GMM}(\mu_{1,1},\cdots,\mu_{1,M},\sigma^2_{1,1},\cdots,\sigma^2_{1,M},w_{1,1},\cdots,w_{1,M})$, $Y \sim \texttt{\large GMM}(\mu_{2,1},\cdots,\mu_{2,N},\sigma^2_{2,1},\cdots,\sigma^2_{2,N},w_{2,1},\cdots,w_{2,N})$ and scalar $a,b$.
\end{lemma}
\begin{proof}
    The PDF of $Z$ can be rewritten as \eqref{th_gmm_linear_pdf}, where $\sum\nolimits_{i=1}^{M} \sum\nolimits_{j=1}^{N} w_{1,i}w_{2,j} = 1$. The PDF of $Z$ is also in the form of a GMM.
\begin{figure*}
\begin{align} 
    &f_Z(z) \nonumber\\
    =&\frac{1}{b} \int_{-\infty}^{+\infty} f(x;\mu_{1,1},\cdots,\mu_{1,M},\sigma_{1,1}^2,\cdots,\sigma_{1,M}^2, w_{1,1} ,\cdots,w_{1,M}) f(\frac{z-ax}{b};\mu_{2,1},\cdots,\mu_{2,N},\sigma_{2,1}^2,\cdots, \sigma_{2,N}^2, w_{2,1},\cdots,w_{2,N})\text{d}x \nonumber \\
    =& \frac{1}{b}\int_{-\infty}^{+\infty} \sum\nolimits_{i=1}^{M} w_{1,i} g(x;\mu_{1,i},\sigma^2_{1,i}) \sum\nolimits_{j=1}^{N} w_{2,j} g(\frac{z-ax}{b}; \mu_{2,j},\sigma^2_{2,j})\text{d}x \nonumber\\         
    =& \frac{1}{b}\sum\nolimits_{i=1}^{M} \sum\nolimits_{j=1}^{N} w_{1,i}w_{2,j} \int_{-\infty}^{+\infty} g(x;\mu_{1,i},\sigma^2_{1,i})g(\frac{z-ax}{b}; \mu_{2,j},\sigma^2_{2,j}) \text{d}x \label{th_gmm_linear_pdf} \\
    =& \sum\nolimits_{i=1}^{M} \sum\nolimits_{j=1}^{N} w_{1,i}w_{2,j} g(x;a\mu_{1,i}+b\mu_{2,j}, a^2\sigma_{1,i}^2+b^2\sigma_{2,j}^2). \nonumber
\end{align}
\hrulefill
\end{figure*}
\end{proof}
Specifically, Gaussian distribution is a GMM with only one component. 
Therefore, the linearity additivity of Gaussian distribution and GMM is given without proof.
\begin{corollary} \label{th_gauss_linear}
    $Z=aX+bY \sim \texttt{\large GMM}(a\mu_1+b\mu_k,\cdots,a\mu_N+b\mu_k,a^2\sigma_1^2+b^2\sigma^2_k,\cdots,a^2\sigma_N^2+b^2\sigma^2_k, w_1,\cdots,w_N)$,     
    when given $X \sim \texttt{\large GMM} (\mu_1,\cdots,\mu_N,\sigma_1^2,\cdots,\sigma_N^2,w_1, \cdots,w_N), Y \sim \mathcal{N}(\mu_k,\sigma^2_k)$ and scalar $a,b$.
\end{corollary}

\begin{lemma}
    \label{lem:r^k-GMM}
    For $k \in \{0,\cdots, K\}$, if $r \sim \texttt{\large GMM}(\cdot)$, $r^{<k>} \sim \texttt{\large GMM}(\cdot)$.
\end{lemma}
\begin{proof}
    We prove the lemma by induction. Naturally, by the procedures in DM, $r^{<K>} \sim \mathcal{N}(0,1)$. Afterwards, based on \eqref{dm_reverse_r0_rk}, $r^{<K-1>}$ is expressed as a linear combination of $r \sim \texttt{\large GMM}(\cdot)$ and Gaussian distributions (i.e., $r^{<K>} \sim \mathcal{N}(0,1)$, $\epsilon_{\Theta}$ and $ \epsilon$). Therefore $r^{<k-1>} \sim \texttt{\large GMM}(\cdot)$. Furthermore, if $k<K$, $r^{<k-1>} $ also obeys $ \texttt{\large GMM}(\cdot)$, which is represented linearly by $r,r^{<k>} \sim \texttt{\large GMM}(\cdot)$ and $\epsilon_{\Theta}, \epsilon \sim \mathcal{N}(\cdot)$. Thus $r^{<0>} \sim \texttt{\large GMM}(\cdot)$ according to mathematical induction \cite{dubinsky1989teaching}.
\end{proof}
Next, we give the proof of Theorem \ref{them_approximation_error}. 

\begin{proof}
Theorem \ref{th_gmm_linear} and Corollary \ref{th_gauss_linear} indicate that the samples generated by DM are more inclined to satisfy GMM, as GMM and Gaussian distribution are both additive. Linear additivity ensures that the approximation error between $\mathcal{D}$ and GMM does not further expand along with the diffusion process. The expected upper limit of this approximation error will be quantified below.

As there exist a bounded approximation error $\xi$ for $r \sim \texttt{GMM}(\cdot)$, we have $r = {r}_{\texttt{GMM}} + \xi$, where ${r}_{\texttt{GMM}} \sim \texttt{\large GMM}(\cdot), r \sim \mathcal{D}$. Let
\begin{align} \nonumber
    \Gamma_{2,1}(k)=
    \begin{cases}
    \frac{\sqrt{1-\upsilon_k}}{\sqrt{\upsilon_k}}\epsilon, & k=1;\\
    \frac{1}{\sqrt{\upsilon_k}} \frac{1-\upsilon_k}{{1-\bar{\upsilon}_k}} \Big(\sqrt{\bar{\upsilon}_k}{r}_{\texttt{GMM}} + \\ \nonumber
    \ {\sqrt{1-\bar{\upsilon}_k}}\xi_1(k)\Big)+\sqrt{\frac{1-\bar{\upsilon}_{k-1}}{1-\bar{\upsilon}_k} \omega_k} \epsilon, \quad & k=2,\cdots,K,
    \end{cases}
\end{align} \nonumber
and
\begin{align} \nonumber
    \Gamma_{2,2}(k)=
    \begin{cases}
    0, & k=1;\\
    \frac{1}{\sqrt{\upsilon_k}} \frac{1-\upsilon_k}{{1-\bar{\upsilon}_k}} \sqrt{\bar{\upsilon}_k}\xi, \quad & k=2,\cdots,K,
    \end{cases}
    % ,
\end{align}

\noindent we have $\Gamma_2(k) = \Gamma_{2,1}(k)+\Gamma_{2,2}(k)$.
According to \eqref{dm_reverse_general_terms}, %we have :\eqref{eq_r_0} on Page \pageref{eq_r_0}.
% \begin{figure*}
        \begin{align}
        % \label{eq_r_0}
        {r}^{<0>}
        = &\underbrace{{r}^{<K>}\prod\nolimits_{i=1}^{K}\Gamma_1(i)}_{{r}^{<0>}_1} + \underbrace{\big[\sum\nolimits_{i=1}^{K}\Gamma_{2,1}(i)\prod\nolimits_{j=1}^{i-1}\Gamma_1(j)\big]}_{{r}^{<0>}_2}+\\ \nonumber
        &\underbrace{\big[\sum\nolimits_{i=1}^{K}\Gamma_{2,2}(i)\prod\nolimits_{j=1}^{i-1}\Gamma_1(j)\big]}_{{r}^{<0>}_3}. \nonumber
    \end{align} 
    % \hrulefill
% \end{figure*}

Based on Lemma \ref{th_gmm_linear}, Corollary \ref{th_gauss_linear}, and Lemma \ref{lem:r^k-GMM}, ${r}^{<0>}_1+{r}^{<0>}_2 \sim \texttt{\large GMM}(\cdot)$. Meanwhile, $r^{<0>}_3$ resembles the approximation error for $r^{<0>}$. On condition of the aforementioned settings,
\begin{align} \nonumber
    \mathbb{E}({r}^{<0>}_3)
    \le& 0.9997 \mathbb{E}(\xi), 
\end{align}
when $K=25$. This means the expected approximation error for the sample generated by DM will be less than that for the original sample. 
\end{proof}

\noindent\emph{Remark}: 
When $K$ is smaller, the expected approximation error for the generated sample will be closer to that for the original one. 
But the condition for the reverse diffusion process in DM is $\mathcal{D}_p(r^{<K>})\sim \mathcal{N}(0,1)$ given $K \rightarrow +\infty$. 
If $K$ is too small, the condition is no longer satisfied.
% the difference between $\mathcal{D}_p(r^{<K>})$ and standard Gaussian distribution is too large, which can lead to a decrease in the accuracy of the generated data. 
Especially in complex multi-source noise, generation error tends to amplify as shown in Fig. \ref{DG_fig} when $K=5$. Considering such a trade-off problem, we set $K=25$ in our experiments. 

\section{DNN Structure and Hyperparameter Configurations}\label{hyper_parameters_setting}
\blus{
In the experiment, the NDD is trained based on MAPPO \cite{yu2022surprising} with two DNNs, that is, the actor network and the critic network. The former consists of $3$-layer recurrent neural networks (RNN) and the latter is a $3$-layer MLP. Both networks update parameters alternately, following the CTDE framework. In addition, policies for agents are from actor networks, which can be optimized by generalized advantage estimation (GAE)\cite{yu2022surprising, schulman2017proximal}.}

\blus{In addition to configurations about noises in MPE and SMAC in Table \ref{tab:noise_config_along_scene}, the number of agents in Adversary, Crypto, Speaker Listener, Spread, and Reference is 3, 2, 2, 3, and 3, respectively. In SMAC, the number and types of agents are included in tasks' names. For example, ``3m" means 3 Marines in MAS, and ``z, s, sc" in tasks' names are ``Zealot, Stalker, Spine Crawler", respectively. Other hyperparameter settings used in NDD are provided in Table \ref{hyper_parameters_tab}.}

\blus{In DM settings, consistent with Theorem \ref{them_generation_error} and Theorem \ref{them_approximation_error}, $K = 25$, $[\omega_1, \cdots, \omega_K]= 0.499\times10^{-2} \times \frac{1}{1+e^{-\boldsymbol{h}}}+10^{-5}$, where $\boldsymbol{h}$ is an array with a length of $K$, evenly spaced in $[-6,6]$. As for distortion risk functions, Table \ref{riskfunnotations} shows the details of $\eta$.} 

% \bibliographystyle{IEEEtran}
% \bibliography{reference}
}